\documentclass[11pt]{article} 

\usepackage{fullpage}

\usepackage{amsmath,amsthm, amssymb}
\usepackage{latexsym}
\usepackage{graphicx}
\usepackage{ifthen}
\usepackage{subcaption}
\usepackage{multicol}
\usepackage{algorithm, algorithmic}
\usepackage[numbers]{natbib}
\usepackage{mathtools}
\usepackage{dsfont}
\usepackage{nicefrac}
\usepackage{wrapfig}
\usepackage{enumitem}
\usepackage[affil-it]{authblk}

\usepackage{hyperref}
\hypersetup{colorlinks=true,linkcolor=blue,citecolor=blue,urlcolor=blue}

\newtheorem*{theorem*}{Theorem}
\newtheorem*{definition*}{Definition}

\usepackage{thm-restate}

\makeatletter

\renewcommand\AB@affilsepx{\qquad \protect\Affilfont}
\makeatother

\renewcommand{\algorithmiccomment}[1]{\bgroup\hfill\footnotesize~#1\egroup}
\newcommand{\Clar}{F_L}
\newcommand{\Ccar}{F_C}
\newcommand{\xlar}{\bx_{L}}
\newcommand{\xcar}{\bx_{C}}
\newcommand{\xlaradv}{\bx^u_{L}}
\newcommand{\xcaradv}{\bx^u_{C}}
\newcommand{\xlaradvadv}{\bx^t_{L}}
\newcommand{\xcaradvadv}{\bx^t_{C}}

\newcommand{\algoneplus}{\textsc{Scar}}
\newcommand{\scar}{\textsc{Scar}}
\newcommand{\talgoneplus}{\textsc{T-Scar}}
\newcommand{\vscar}{\textsc{Vanilla-Scar}}
\newcommand{\pa}{\textsc{Pointwise}}

\newcommand{\simba}{\textsc{SimBA}}

\newcommand{\INPUT}{\item[{\bf Input:}]}

\DeclareMathOperator{\E}{\mathbb{E}}

\newcommand{\zeros}{{\mathbf 0}}

\DeclareMathOperator*{\argmax}{arg\,max}
\DeclareMathOperator*{\argmin}{arg\,min}

\newcommand{\U}{\mathcal{U}}

\newcommand{\bx}{\mathbf{x}}
\newcommand{\be}{\mathbf{e}}

\newcommand{\bg}{\mathbf{g}}

\newcommand{\bw}{\mathbf{w}}

\title{Adversarial Attacks on Binary Image  Recognition  Systems}
\author[1,2]{\qquad \quad \ \ Eric Balkanski} 
\author[1]{Harrison Chase}
\author[1]{Kojin Oshiba}
\author[1]{\newline\\Alexander Rilee}
\author[1,3]{Yaron Singer}
\author[1]{Richard Wang}
\affil[1]{Robust Intelligence}
\affil[2]{Columbia University}
\affil[3]{Harvard University}
\date{}
\begin{document}

\maketitle

\begin{abstract}
	We initiate the study of adversarial attacks on models for binary (i.e. black and white) image classification. Although there has been a great deal of work on attacking models for colored and grayscale images, little is known about attacks on models for binary images. Models trained to classify binary images are used in text recognition applications such as check processing, license plate recognition, invoice processing, and many others. In contrast to colored and grayscale images, the search space of  attacks on binary images is extremely restricted and noise cannot be hidden with minor perturbations in each pixel. Thus, the optimization landscape of attacks on binary images introduces new fundamental challenges.  
	
	In this paper we introduce a new attack algorithm called \algoneplus, designed to fool classifiers of binary images. We show that \algoneplus \ significantly outperforms existing $L_0$ attacks applied to the binary setting and use it to demonstrate the vulnerability of real-world text recognition systems. \algoneplus’s strong performance in practice contrasts with hardness results that show the existence of  classifiers that are provably robust to large perturbations. In many cases, altering a single pixel is sufficient to trick Tesseract, a popular open-source text recognition system, to misclassify a word as a different word in the English dictionary. We also license software from providers of check processing systems to most of the major US banks and  demonstrate the vulnerability of check recognitions for mobile deposits.  These systems are substantially harder to fool since they classify both the handwritten amounts in digits and letters, independently.  Nevertheless, we generalize \algoneplus \ to design attacks that fool state-of-the-art check processing systems using unnoticeable perturbations that lead to misclassification of deposit amounts.  Consequently, this is a powerful method to perform financial fraud. 
\end{abstract}

\newpage
\section{Introduction}

In this paper we study adversarial attacks on models designed to classify binary (i.e. black and white) images.  Models for binary image classification are heavily used across a variety of applications that include receipt processing, passport recognition,  check processing, and license plate recognition, just to name a few.  In such applications, the text recognition system typically binarizes the input image (e.g. check processing~\cite{check_binarization}, document extraction~\cite{document_extraction}) and trains a model to classify binary images.

In recent years there has been an overwhelming interest in understanding the vulnerabilities of AI systems.  In particular,  a great deal of work has designed attacks on image classification models (e.g. ~\cite{szegedy13,fgsm, deepfool, bim, lzeropapernot, pgd, lzerocarlini,  bbox1, bboxmadry, bboxmadry2,  bbox2, simba, nattack}).  Such attacks distort images in a manner that is virtually imperceptible to the human eye and yet cause state-of-the-art models to misclassify these images. Although there has been a great deal of work on attacking image classification models, these attacks are designed for colored and grayscale images. These attacks hide the noise in the distorted images by making minor perturbations in the color values of each pixel. 

Somewhat surprisingly, when it comes to binary images, the vulnerability of state-of-the-art models is poorly understood.  In contrast to colored and grayscale images, the search space of  attacks on binary images is extremely restricted and noise cannot be hidden with minor perturbations of color values in each pixel. As a result, existing attack algorithms on machine learning systems do not apply to binary inputs.  Since binary image classifiers are used in  high-stakes decision making and are heavily used in banking and other multi-billion dollar industries, the natural question is:

\begin{center}
	\emph{Are models for binary image classification used in industry vulnerable to adversarial attacks?}
\end{center}

In this paper we initiate the study of attacks on binary image classifiers. We develop an attack algorithm, called \scar, designed to fool binary image classifiers. \scar \ carefully selects pixels to flip to the opposite color in a query efficient manner, which is a central challenge when attacking black-box models. 

We first show that \scar \ outperforms existing attacks that we apply to the binary setting on multiple models trained over the MNIST and EMNIST datasets, as well as models for handwritten strings and printed word recognition.  The most relevant attacks for binary images are  $L_0$ attacks which minimize the number of pixels in which the distorted image differs from the original image \cite{lzeropapernot, lzerocarlini, pointwise, simba}.  We then use \scar \ to demonstrate the vulnerability of  text recognition systems used in industry. We fool commercial check processing systems used by most of the major US banks for mobile check deposits.  One major challenge in attacking these systems, whose software we licensed from providers, is that there are two independent classifiers, one for the amount written in words and one for the amount written in numbers, that must be fooled with the same wrong amount. Check fraud is a major concern for US banks,  accounting for $\$1.3$ billion in losses in 2018~\cite{checkreport}. Since check fraud occurs at large scale, we believe that the vulnerability of check processing systems to adversarial attacks raises a serious concern.

We also show that no attack can obtain reasonable guarantees on the number of pixel inversions needed to cause misclassification as there exist simple classifiers that are provably robust to even large perturbations. There exist classifiers for $d$-dimensional binary images such that every class contains some image that requires $\Omega(d)$ pixel inversions ($L_0$ distance) to change the label of that image and such that for every class,  a random image in that class requires $\Omega(\sqrt{d})$ pixel inversions in expectation.

\paragraph{Related work.} The study of adversarial attacks was initiated in the seminal work by 
\citet{szegedy13} that showed that models for image classification are  susceptible to minor perturbations in the input. There has since then been a long line of work developing attacks on colored and greyscale images.  Most relevant to us are $L_0$ attacks, which iteratively make minor perturbations in carefully chosen pixels to minimize the total number of pixels that have been modified \cite{lzeropapernot, lzerocarlini, pointwise, simba}. We compare our attack to two $L_0$ attacks that are applicable in the black-box binary setting \cite{pointwise, simba}. Another related area of research focuses on developing attacks that query the model as few times as possible \cite{bbox1, bboxmadry, bboxmadry2, simba, nattack, bbox2, al2019there}.  We discuss  below why most of these attacks  cannot be applied to the binary setting. There has been previous work on attacking  OCR systems \cite{ocr}, but the setting deals with grayscale images and white-box access to the model.

Attacks on colored and grayscale images employ continuous optimization techniques and are fundamentally different than attacks on binary images which, due to the binary nature of each pixel, employ combinatorial optimization approaches. Previous work has formulated adversarial attack settings as combinatorial optimization problems, but in drastically different settings. \citet{combi_dimakis} consider attacks on text classification for tasks such as sentiment analysis and fake news detection, which is a different domain than OCR. \citet{combi-colored} formulate $L_\infty$ attacks on colored image classification as a combinatorial optimization problem where the search space for the change in each pixel is $\{-\varepsilon, \varepsilon\}$ instead of $[-\varepsilon, \varepsilon]$.

Finally, we also note that binarization, i.e. transforming colored or grayscale images into black and white images, has been studied as a technique to improve the robustness of models, especially to $L_2$ and $L_\infty$ attacks \cite{pointwise, madry_binarization, ding2019sensitivity}.

\paragraph{Previous attacks are ineffective in the binary setting.} Previous attacks on classifiers for grayscale and colored images are not directly applicable to classifiers for binary images. These attacks iteratively cause small perturbations in the pixel values. Small changes in the pixel values of binary images are not possible since there are only two possible values.

One potential approach to extend previous attacks to the binary setting  is to relax the binary pixel values to be in the grayscale range and run an attack over this relaxed domain. The issue with this approach is that binary image classifiers only take as input binary images and small changes in the relaxed grayscale domain are lost when rounding the pixel values back to being binary. 

Another potential approach  is to increase the step size of an attack such that a small change in the pixel value of a grayscale or colored image instead causes the pixel value of a binary image to  flip. For gradient-based attacks, which cause minor perturbations in all pixels, this  causes a large number of pixel flips.  This approach is more relevant for $L_0$ attacks since they  perturb a smaller number of pixels. There are two $L_0$ attacks which can be applied to the binary setting with this approach. The first is \simba~\cite{simba}  which modifies a single pixel at each iteration and the second is \pa~ \cite{pointwise}, which first applies random salt and pepper noise until the image is misclassified and then greedily returns each modified pixels to its original color if the image remains misclassified.  However, even these $L_0$ attacks extended to the binary setting result in a large and visible number of pixel inversions, as shown in the experiments in Section~\ref{sec:experiments}.

\section{Problem Formulation}
\label{sec:problem}
\paragraph{Binary images.}
We consider binary images $\bx \in \{0, 1\}^{d}$, which are $d$-dimensional  images such that each pixel  is either black (value $0$) or white (value $1$). An $m$-class classifier $F$ maps $\bx$ to a probability distribution $F(\bx) \in [0, 1]^m$ where $F(\bx)_{i}$ corresponds to the confidence that image $\bx$ belongs to class $i$. The predicted label $y$ of $\bx$ is the class with the highest confidence, i.e., $y = \argmax_{i}  F(\bx)_{i}$.

\paragraph{OCR systems.} Optical Character Recognition (OCR) systems convert images of handwritten or printed text to strings of characters.  Typically, a preprocessing step of OCR systems  is to convert the input to a binary format. To formalize the problem of attacking OCR systems, we consider a classifier $F$ where the labels are strings of characters. Given a binary image $\bx$ with label $y$, we wish to produce an adversarial example $\bx'$ which is similar to $\bx$, but has a predicted label $y' \neq y$. For example, given an image $\bx$ of license plate \textsf{23FC6A}, our goal is to produce a similar image $\bx'$ that is recognized as a different  license plate number. We measure the similarity of an adversarial image $\bx'$ to the original image $\bx$ with a perceptibility metric $D_{\bx}(\bx')$. For binary images, a natural metric is the number of pixels where $\bx$ and $\bx'$ differ, which corresponds to the $L_{0}$ distance between the two images. Finding an adversarial example can thus be formulated as the following optimization problem:
\begin{align*}
	\label{eq:opt}
	&\min_{\substack{\bx' \in \{0, 1\}^{d}\\ \|\bx - \bx'\|_{0} \leq k}}  F(\bx')_{y}
\end{align*}
where $k$ is the maximum dissimilarity tolerated for adversarial image $\bx'$.
For targeted attacks with target label $y_{t}$, we instead maximize $F(\bx')_{y_{t}}$. Since there are at least $\binom{d}{k}$ feasible solutions for $\bx'$, which is exponential in $k$, this is a computationally hard problem.

\paragraph{Check processing systems.} A check processing system $F$ accepts as input a binary image $\bx$ of a check and outputs confidence scores $F(\bx)$ which represent the most likely amounts that the check is for.
Check processing systems are a special family of OCR systems that consist of two independent models that verify each other.  Models  $\Ccar$ and  $\Clar$  for Courtesy and Legal Amount Recognition (CAR and LAR) classify the amounts written in numbers and in words respectively. If the predicted labels of the two models do not match, the check is flagged. For example, if the CAR of a valid check  reads $100$, and the LAR  reads ``one hundred'', the two values match and the check is processed. The main  challenge with attacking check processing systems over an input $\bx$ is to craft an adversarial example $\bx'$ with the same target label for both  $\Ccar$ and $\Clar$. Returning to the previous example, a successful adversarial check image might have the CAR read $900$ and the LAR read ``nine hundred''. For this targeted attack, the  optimization problem is:
\begin{align*}
	\max_{\substack{\bx'  \in \{0, 1\}^{d},  y_{t} \neq y \\  ||\bx - \bx'||_0 \leq k}} \quad & \Ccar(\bx')_{y_{t}} + \Clar(\bx')_{y_{t}} \\
	\text{subject to}\quad & y_{t}  = \text{argmax}_{i} \Ccar(\bx')_{i}  =  \text{argmax}_{i} \Clar(\bx')_{i} 
\end{align*}
The attacker first needs to select a target amount $y_{t}$ different from the true amount $y$, and then attack  $\Ccar$ and $\Clar$ such that both misclassify $\bx'$ as amount $y_{t}$. Since check processing systems also flag checks for which the models have low confidence in their predictions, we want to maximize both the probabilities $\Ccar(\bx')_{y_{t}}$ and $\Clar(\bx')_{y_{t}}$. In order to have $\bx'$ look as similar to $\bx$ as possible, we also limit the number of modified pixels to be at most $k$. Check processing systems are configured such that $\Ccar$ and $\Clar$ only output the probabilities for a limited number of their most probable amounts. This limitation makes the task of selecting a target amount challenging, as aside from the true amount, the most probable amounts for each of $\Ccar$ and $\Clar$ may be disjoint sets.

\paragraph{Black-box access.} We assume that we do not have any information about the OCR model $F$  and can only observe its outputs, which we formalize with the score-based black-box setting where an attacker only has access to the output probability distributions of a model $F$ over queries $\bx'$.

\section{Existence of Provably Robust Classifiers for Binary Images}
\label{sec:robust}
Before presenting the attacks, we first show the existence of simple binary image classifiers that are provably robust to any attack that modifies a large, bounded, number of pixels. This implies that there is no attack that can obtain reasonable guarantees on the number of pixel inversions ($L_0$~distance) needed to cause misclassification.

We show such results for two different settings. First,  there exists an $m$-class linear classifier $F$ for $d$-dimensional binary images such that every class contains some image whose predicted label according to $F$ cannot be changed with $o(d)$ pixel flips, i.e., every class contains at least one image which requires a number of pixel flips that is linear in the total number of pixels to be attacked. The analysis uses a probabilistic argument and is deferred to the appendix.
\begin{restatable}{rThm}{thmone}
	There exists an $m$-class linear classifier $F$ for $d$-dimensional binary images s.t. for all classes $i$, there exists at least one binary image $\bx$ in $i$ that is robust to $d/4 - \sqrt{2d \log m}/2$ pixel changes, i.e., for all $\bx'$ s.t. $\|\bx - \bx'\|_0   \leq d/4 - \sqrt{2d \log m}/2$, $\argmax_j F(\bx')_j = i$. 
\end{restatable}

This robustness result holds for all $m$ classes, but only for the most robust image in each  class.  We
also show the existence of a classifier robust to attacks on an image drawn uniformly at random. There exists a $2$-class classifier such that for both classes, a uniformly random image in that class requires, in expectation, $\Omega(\sqrt{d})$ pixel flips to be attacked. The analysis relies on anti-concentration bounds and is deferred to the appendix as well. 

\begin{restatable}{rThm}{thmtwo}
	There exists a $2$-class linear classifier $F$ for $d$-dimensional binary images such that for both classes $i$, a uniformly random binary image $\bx$ in that class $i$  is robust to $\sqrt{d}/8$ pixel changes in expectation, i.e.
	$\E_{\bx \sim \U(i)}[\min_{\bx' : \argmax_j F(\bx')_j \neq i} \|\bx - \bx'\|_0] \geq  \sqrt{d}/8.$
\end{restatable}

These hardness results hold for worst-case classifiers. Experimental results in Section~\ref{sec:experiments} show that, in practice, classifiers for binary images  are highly vulnerable and that the algorithms that we present next require a small number of pixel flips to cause misclassification.

\section{Attacking Binary Images}
In this section, we present \scar, our main attack algorithm. We begin by describing a simplified version of \scar, Algorithm~\ref{alg:one},  then discuss the issues of hiding noise in binary images and optimizing the number of queries, and finally describe \scar.  At each iteration, Algorithm~\ref{alg:one} finds the pixel $p$ in input image $\bx$ such that flipping $x_{p}$ to the opposite color causes the largest decrease in $F(\bx')_y$, which is the confidence  that this perturbed input $\bx'$ is classified as the true label $y$. It flips this pixel and repeats this process until either the perturbed input is classified as label $y' \neq y$ or the maximum $L_0$ distance $k$ with the original image is reached. Because binary images $\bx$ are such that $\bx \in \{0,1\}^d,$ we implicitly work in $\mathbb{Z}_2^d$. In particular, with $\be_1, \hdots, \be_d$ as the standard basis vectors, $\bx’ + \be_{p}$ represents the image $\bx'$ with pixel $p$ flipped.
\vspace{-.1cm}
\begin{algorithm}[H]
	\caption{A combinatorial attack on OCR systems.}
	\begin{algorithmic}
		\INPUT  model $F$, image $\bx$, label $y$
		\STATE $ \bx' \leftarrow \bx$
		\STATE \textbf{while} $y = \argmax_{i} F(\bx')_{i}$ and $\|\bx' - \bx\|_0 \leq k$ \textbf{do}
		\STATE \qquad $p' \leftarrow \argmin_{p}  F(\bx’ + \be_{p})_y$
		\STATE \qquad $\bx' \leftarrow \bx’ + \be_{p'}$ \\
		\textbf{return} $\bx'$ 
	\end{algorithmic}
	\label{alg:one}
\end{algorithm}
\vspace{-.1cm}

Although the adversarial images produced by Algorithm~\ref{alg:one} successfully fool models and have small $L_{0}$ distance to the original image, it suffers in two aspects: the noise added to the inputs is visible to the human eye, and the required number of queries to the model  is large.

\paragraph{Hiding the noise.} Attacks on images in a binary domain are fundamentally different from attacks on colored or grayscale images. In the latter two cases, the noise is often imperceptible because the 
change to any individual pixel is small relative to the range of possible colors. Since attacks on binary images can only invert a pixel's color or leave it untouched,  noisy pixels are highly visible if their colors contrast with that of their neighboring pixels. This is a shortcoming of Algorithm~\ref{alg:one}, which results in noise with small $L_{0}$ distance but that is highly visible (for example, see Figure~\ref{fig:examples}). To address this issue, we impose a new constraint which only allows modifying pixels on the \textit{boundary} of black and white regions in the image. A pixel is on a boundary if it is white and at least one of its eight neighboring pixels is black (or vice-versa). Adversarial examples produced under this constraint have a greater $L_{0}$ distance to their original images, but the noise is significantly less noticeable.

\paragraph{Optimizing the number of queries.} An attack may be computationally expensive if it requires many queries to a black-box model. For paid services where a model is hidden behind an API, running attacks can be financially costly as well. Several works have proposed techniques to reduce the number of queries. Many of these are based on gradient estimation \cite{bbox1, bbox2, bboxmadry, bboxmadry2, al2019there}. Recently, several gradient-free black-box attacks have also been proposed. \citet{nattack} and \citet{combi-colored} propose two such approaches, but these rely on taking small steps of size $\varepsilon$ in a direction which modifies \emph{all} pixels.  Taking small steps in some direction, whether with gradient estimation or gradient-free approaches, is not possible here  since there are only two possible values for each pixel. \simba  \cite{simba}, another gradient-free attack, can be extended to the binary setting and is evaluated in the context of binary images in Section~\ref{sec:experiments}. We propose two optimization techniques to exploit correlations between pixels both spatially and temporally. We define the \textit{gain} from flipping pixel $p$ at point $\bx'$ as the following discrete derivative of $F$ in the direction of $p$:
\[F(\bx')_{y} - F(\bx'+ \be_{p})_{y}\]
We say a pixel $p$ has \textit{large gain} if this value is larger than a threshold $\tau$.
\begin{itemize}
	\item \textbf{Spatial correlations.}  Pixels in the same spatial regions are likely to have similar discrete derivatives (e.g. Figure~4 in appendix). At every iteration, we prioritize evaluating the gains of the eight pixels $N(p)$ neighboring the pixel $p$ which was modified in the previous iteration of the algorithm. If one of these pixels has large gain, then we flip it and proceed to the next iteration without evaluating the remaining pixels.
	
	\item \textbf{Temporal correlations.} . Pixels with large discrete
	derivatives at one iteration are  likely to also have large
	discrete derivatives in the next iteration  (e.g. Figure~5 in appendix). At each iteration, we first consider pixels that had large gain in the previous iteration. If one of these pixels still produces large gain in the current iteration, we flip it and proceed to the next iteration without evaluating the remaining pixels. 
\end{itemize}

\paragraph{\algoneplus.} In order to improve on the number of queries, \scar \ (Algorithm~\ref{alg:full}) prioritizes evaluating the discrete derivatives at pixels which are expected to have large gain according to the spatial and temporal correlations. If one of these pixels has large gain, then it is flipped and the remaining pixels are not evaluated. If none of these pixels have large gain, we then consider all pixels on the boundary $B(\bx)$ of black and white regions in the image $\bx$. In this set, the pixel with the largest gain is flipped regardless of whether it has gain greater than $\tau$. As before, we denote the standard basis vector in the direction of coordinate $i$ with $\be_i$. We keep track of the gain of each pixel with vector $\bg$.

\begin{algorithm}[H]
	\caption{\algoneplus, Shaded Combinatorial Attack on Recognition sytems.}
	\begin{algorithmic}
		\INPUT  model $F$, image $\bx$, label $y$, threshold $\tau$, budget $k$
		\STATE $ \bx' \leftarrow \bx, \bg \leftarrow \zeros$
		\STATE \textbf{while}  $y = \argmax_{i} F(\bx')_{i}$ and $\|\bx' - \bx\|_0 \leq k$ \textbf{do}
		\STATE \qquad \textbf{for} $p: g_p \geq \tau \text{ or } p \in N(p')$ \textbf{do}
		\STATE \qquad\qquad $g_p \leftarrow F(\bx')_y - F(\bx' + \be_p)_y$
		\STATE \qquad \textbf{if} $\max_{p} g_p < \tau$ \textbf{then} 
		\STATE \qquad\qquad \textbf{for} $p \in B(\bx')$ \textbf{do}
		\STATE \qquad\qquad\qquad $g_p \leftarrow F(\bx')_y - F(\bx' + \be_p)_y$
		\STATE \qquad $p' \leftarrow \argmax_{p} g_p$
		\STATE \qquad $\bx' \leftarrow \bx' + \be_{p'}$ \\
		\textbf{return} $\bx'$ 
	\end{algorithmic}
	\label{alg:full}
\end{algorithm}

Algorithm~\ref{alg:full} is an untargeted attack which finds $\bx'$ which is classified as label $y' \neq y$ by $F$. It can easily be modified into a targeted attack with target label $y_{t}$ by changing the first condition in the while loop from $y = \argmax_{i} F(\bx')_{i}$ to $y_t \neq \argmax_{i} F(\bx')_{i}$  and by computing the gains $g_p$  as $F(\bx + \be_{p})_{y_t} - F(\bx)_{y_t}$ instead of $F(\bx)_y - F(\bx + \be_{p})_y$. Even though \algoneplus \ performs well in practice, we show in the appendix that there exists simple classifiers $F$ for which \algoneplus  \ requires, for most images $\bx$, a linear number $k = O(d)$ of  pixel flips to find an adversarial example $\bx'$.

\section{Simultaneous Attacks}
There are two significant challenges to attacking check processing systems. In the previous section, we discussed the challenge caused by the preprocessing step that binarizes check images~\cite{check_binarization}. The second challenge is that check processing systems employ two independent models that verify the output of the other model: $\Ccar$ classifies the amount written in numbers, and $\Clar$ classifies the amount written in letters. We thus propose an algorithm which tackles the problem of attacking two separate OCR systems simultaneously.  

A natural approach is to search for a target amount at the intersection of what $\Ccar$ and $\Clar$ determines are probable amounts. However, on unmodified checks, the models are often highly confident of the true amount, and other amounts have extremely small probability, or do not even appear at all as predictions by the models. 

To increase the likelihood of choosing a target amount which will result in an adversarial example, we first proceed with an untargeted attack on both $\Ccar$ and $\Clar$ using \scar, which returns image $\bx^u$ with reduced confidence in the true amount $y$. Then we choose the target amount $y_t$ to be the amount $i$ with the maximum value $\min(\Ccar(\bx^u)_{i}, \Clar(\bx^u)_{i})$, since our goal is to attack both $\Ccar$ and $\Clar$. Then we run \talgoneplus, which is the targeted version of \scar, twice to perform targeted attacks on both $\Ccar$ and $\Clar$ over image $\bx^u$.

\begin{algorithm}[H]
	\caption{The attack on check processing systems.}
	\begin{algorithmic}
		\INPUT check image $\bx$, models $\Ccar$ and $\Clar$, label $y$
		\STATE  $\xcar, \xlar \leftarrow$ extract CAR and LAR regions of $\bx$
		\STATE $\xcaradv, \xlaradv  \leftarrow \algoneplus(\Ccar, \xcar), \algoneplus(\Clar, \xlar)$
		\STATE $y_t \leftarrow \max_{i \neq y} \min(\Ccar(\xcaradv)_{i}, \Clar(\xlaradv)_{i})$
		\STATE $\xcaradvadv, \xlaradvadv \leftarrow \talgoneplus(\Ccar, \xcaradv, y_t), \talgoneplus(\Clar, \xlaradv, y_t)$ 
		\STATE $\bx^t \leftarrow$ replace CAR, LAR regions of $\bx$ with $\xcaradvadv, \xlaradvadv$
		\STATE \textbf{return} $\bx^t$
	\end{algorithmic}
	\label{alg:check}
\end{algorithm}

\section{Experiments} \label{sec:experiments}
We demonstrate the effectiveness of \textsc{Scar} for attacking text recognition systems.  We  attack, in increasing order of model complexity, standard models  for single handwritten character classification (Section~\ref{sec:mnist}),   an LSTM model for handwritten numbers classification (Section~\ref{sec:lstm}),  a widely used open source model for typed text recognition called Tesseract (Section~\ref{sec:tesseract}), and finally commercial check processing systems used by banks for mobile check deposit (Section~\ref{sec:checkresults}).

\subsection{Experimental setup}

\paragraph{Benchmarks.} We compare four attack algorithms.

\begin{itemize}
	\item \textbf{\scar}, which is  Algorithm~\ref{alg:full} with threshold $\tau = 0.1$.
	\item \textbf{\vscar}, which  is Algorithm~\ref{alg:one}. We compare \scar\ to Algorithm~\ref{alg:one} to demonstrate the importance of hiding the noise and optimizing the number of queries.
	\item \textbf{\simba}, which is Algorithm 1 in \cite{simba} with the Cartesian basis and $\varepsilon = 1$. \simba\ is an algorithm for attacking (colored) images in black-box settings using a small number of queries. At every iteration, it samples a direction $\mathbf{q}$ and takes a step towards $\varepsilon \mathbf{q}$ or $- \varepsilon \mathbf{q}$ if one of these improves the objective. In the setting where $\mathbf{q}$ is sampled from  the Cartesian basis and $\varepsilon = 1$, \simba\ corresponds to an $L_{0}$ attack on binary images which iteratively chooses a random pixel and flips it if doing so results in a decrease in the confidence of the true label.
	\item \textbf{\pa} \cite{pointwise} first applies random salt and pepper noise until the image is misclassified. It then greedily returns each modified pixel to its original color if the image remains misclassified. We use the implementation of this attack available in Foolbox \cite{rauber2017foolbox}.
\end{itemize}

\paragraph{Metrics.} To evaluate the performance of each attack $A$ over a model $F$ and test set $X$, we use three metrics.

\begin{itemize}
	\item The \textbf{success rate} of $A$ is the fraction of images $\bx \in X$ for which the output image $\bx' = A(\bx)$ is adversarial, i.e. the predicted label $y'$ of $\bx'$ is different from the true label $y$ of $\bx$. We only attack images $\bx$ which are initially correctly classified by $F$.
	
	\item We use the $\mathbf{L_0}$ \textbf{distance} to measure how similar an image $\bx' = A(\bx)$ is to the original image $\bx$, which is the number of pixels where $\bx$ and $\bx'$ differ. 
	
	\item The \textbf{number of queries} to model $F$ to obtain output image $\bx' = A(\bx)$. 
\end{itemize}

\paragraph{The distance constraint $k$.} Because the image dimension $d$ differs for each experiment, we seek a principled approach to selecting the maximum $L_{0}$ distance $k$. For an image $\bx$ with label $y$, the $L_0$ constraint is
\[k = \alpha \cdot \frac{\mathcal{F}(\bx)}{|y|},\]
where $\mathcal{F}(\bx)$ counts the number of pixels in the foreground of the image, $\alpha \in [0, 1]$ is a fixed fraction, and $|y|$ represents the number of characters in $y$, e.g. $|\textsf{23FC6A}| = 6$. In other words, $k$ is a fixed fraction of the average number of pixels per character in $\bx$. In our experiments, we set $\alpha = \frac{1}{5}$.

\subsection{Digit and character recognition systems}
\label{sec:mnist}

\begin{figure*}
	\centering
	\includegraphics[width=.99\linewidth]{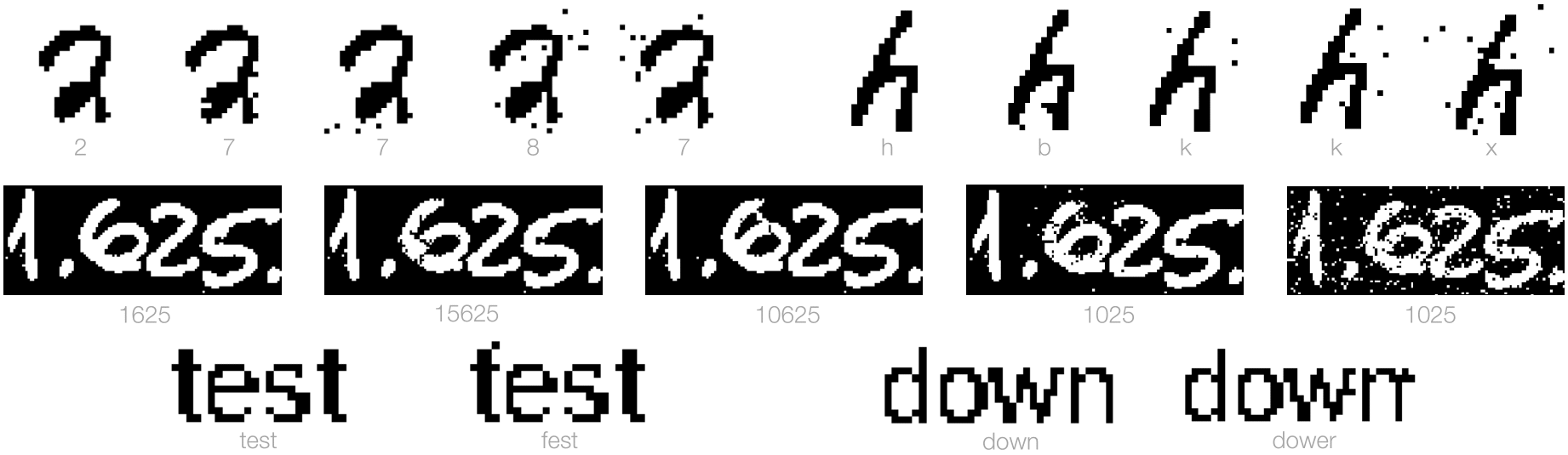}
	\caption{Examples of attacks on a CNN trained over MNIST (top left), a CNN trained over EMNIST (top right), an LSTM for handwritten numbers (center), and Tesseract for typed words (bottom). The images correspond to, from left to right, the original image, the outputs of \scar, \vscar, \pa, and \simba. The predicted labels are in light gray below each image. For Tesseract attacks (bottom), we show the original image and \scar 's output.}
	\label{fig:examples}
\end{figure*}%

For each experiment, we provide further details about the datasets and models in the appendix.

\paragraph{The dataset.} We train models over binarized versions of the MNIST digit~\cite{mnist}  and EMNIST letter~\cite{emnist} datasets. We binarize each dataset with the map $x \mapsto \lfloor \frac{x}{128} \rfloor$. We additionally preprocess the EMNIST letter dataset to only include lowercase letters, since an uppercase letter which is misclassified as the corresponding lowercase letter does not change the semantic meaning of the overall word. We randomly select 10 correctly-classified samples  from each class in MNIST and EMNIST lowercase letters to form two datasets to attack.

\paragraph{Models.} We consider five models, trained in the same manner for the MNIST and EMNIST  datasets. Their Top-1 accuracies are given in Table~1.

\begin{table}[H]
	\centering
	\begin{tabular}{c c c}
		\textbf{Model} & \textbf{MNIST} & \textbf{EMNIST}\\ 
		& (Top-1 accuracy) & (Top-1 accuracy)\\ \hline
		LogReg & $0.919$& $0.809$  \\ \hline
		MLP2 & $0.980$ & $0.934$\\ \hline
		CNN & $0.990$ & $0.950$\\ \hline
		LeNet5 & $0.990$ & $0.943$\\ \hline
		SVM & $0.941$ & $0.875$
	\end{tabular}
	\label{t:accuracy}
	\caption{The Top-1 accuracies of a logistic regression model, a 2-layer perceptron, a convolutional neural network,  a neural network from \cite{lecun1998gradient}, and a support vector machine trained and evaluated over the MNIST and EMNIST datasets.}
\end{table}

\paragraph{Results.} We discuss the results of the attacks on the CNN model trained over MNIST and on the LeNet5 model trained over EMNIST, which are representative cases of the results for neural network models. The full results for the remaining 8 models are in the appendix.

In Figure~\ref{fig:plots}, we observe that for fixed $L_0$ distances $\kappa \leq k$, \vscar \ has the highest success rate on the CNN model, i.e. the largest number of successful attacks with an $L_0$ distance at most $\kappa$. For example, $80\%$ of the images were successfully attacked by flipping at most $7$ of the $784$ pixels of an MNIST image. \scar's success rate by $L_0$ distance is very close to  \vscar, but enjoys two main advantages: first, the number of queries needed for these attacks is significantly smaller. Second, as shown  in Figure~\ref{fig:examples}, even though the $L_0$ distance is slightly larger than \vscar, the noise is less visible.  \simba \ requires very few queries to obtain a success rate close to $40\%$ and $65\%$ respectively on the CNN and LeNet5, but this attack results in images with   large $L_0$ distances. The success rate of this attack does not increase past $40\%$ and $65\%$ with a larger number of queries because  the noise constraint $k$ is reached. \pa \ obtains a success rate close to $85\%$ and $98\%$ on the CNN and LeNet5, respectively. The average $L_0$ distance of the images produced by \pa \ is between \scar  \ and \simba. Overall, \scar\ obtains the best number of queries and $L_0$ distance combination. It is the only attack, together with \vscar, which consistently obtains a success rate close to $100$ percent on MNIST and EMNIST models.

\subsection{LSTM on handwritten numbers} \label{sec:lstm}
\paragraph{The dataset.}
We train an OCR model on the ORAND-CAR-A dataset, part of the HDSRC 2014 competition on handwritten strings \cite{diem2014icfhr}. This dataset consists of 5793 images from real bank checks taken from a Uruguayan bank. Each image contains between 2 and 8 numeric characters.

\begin{figure*}
	\centering
	\includegraphics[trim=.5cm 0 0 0, width=.225\linewidth]{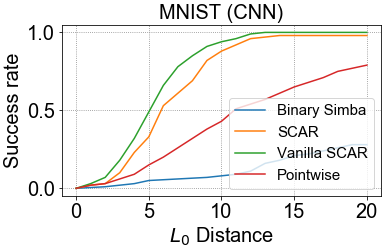}
	\hspace{.2cm}
	\includegraphics[width=.225\linewidth]{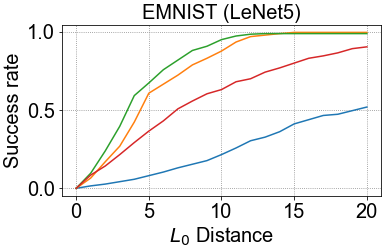}
	\hspace{.2cm}
	\includegraphics[width=.225\linewidth]{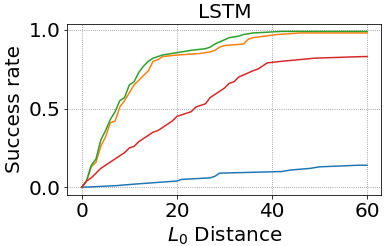}
	\hspace{.2cm}
	\includegraphics[width=.235\linewidth]{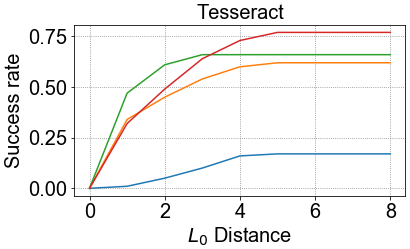}
	\includegraphics[width=.8\linewidth]{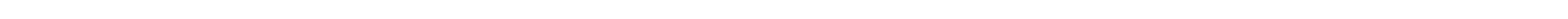}
	\includegraphics[width=.235\linewidth]{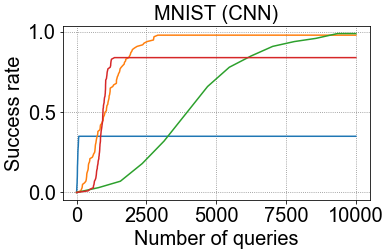}
	\hspace{.2cm}
	\includegraphics[width=.225\linewidth]{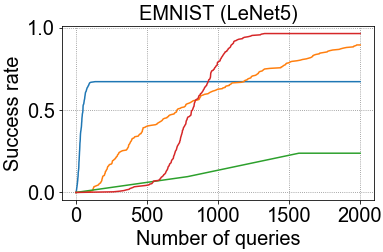}
	\hspace{.2cm}
	\includegraphics[width=.225\linewidth]{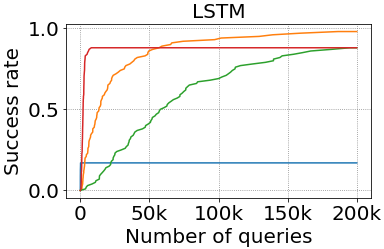}
	\hspace{.2cm}
	\includegraphics[width=.235\linewidth]{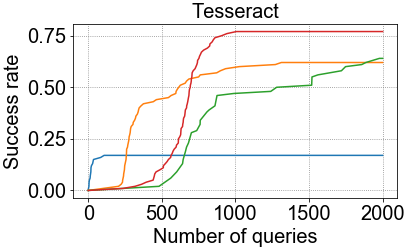}
	\caption{Success rate by $L_0$ distance and by number of queries for a CNN model on MNIST, a LeNet5 model on EMNIST, an LSTM model on handwritten numbers, and the Tesseract model over printed words.}
	\label{fig:plots}
\end{figure*}

\paragraph{The LSTM model.}
We implement the OCR model described in \cite{mor2018confidence}, which consists of a convolutional layer, followed by a 3-layer deep bidirectional LSTM, and optimizes for CTC loss.  The trained model achieves a precision score of $85.7\%$ on the test set of ORAND-CAR-A, which would have achieved first place  in the HDSRC 2014 competition.

\paragraph{Results.} The results  are similar to the attacks on the CNN-MNIST model. \simba \ has a less than $20$ percent success rate. \pa \ obtains a high success rate with a  small number of queries, but is outperformed by \scar \ and \vscar \ in terms of $L_0$ distance. Due to the images being high-dimensional ($d \approx 50,000$) and consisting of multiple digits, the reason why \simba \ performs poorly is because the flipped pixels are spread out over the different digits (see Figure~\ref{fig:examples}).

\subsection{Tesseract on printed words} \label{sec:tesseract}

We show that the vulnerability of OCR systems concerns not only handwritten text, but also printed text, which one might expect to be more robust. We explore this vulnerability in the context of English words and show that in many cases, a change in a single pixel suffices for a word to be misclassified as another word in the English dictionary.

\paragraph{The Tesseract model.} Tesseract is a popular open-source text recognition system that is sponsored by Google \cite{smith2007overview}. We used Tesseract version 4.1.1 trained for the English language. Tesseract 4 is based on an LSTM model (see \cite{ocr} for a detailed description of the architecture of Tesseract's model). We note that Tesseract  is designed for printed text, rather than handwritten.

\paragraph{The dataset.} We attack 100 images of single  printed English words that consist of four characters (the full list of words, together with the labels of the attacked images, can be found in the appendix), chosen randomly among those correctly classified by Tesseract, which has  $96.5\%$ accuracy rate.  For some attacked images with a lot of noise, Tesseract does not recognize any word and rejects the input. Since the goal of these attacks is to misclassify images as words with a different meaning, we only consider an attack to be successful if the adversarial image produced is classified as a word in the English dictionary. For example, consider an attacked image of the word ``one". If Tesseract does not recognize any word in this image, or recognizes ``oe" or ``:one", we do not count this image as a successful attack.

\paragraph{Results.} The main  result for the attacks on Tesseract is that, surprisingly, for around half of the images, flipping a \emph{single} pixel results in the image being classified as a different word in the English dictionary (see Figure~\ref{fig:plots}). \scar \ again produces attacks with $L_0$ distance close to \vscar, but with fewer queries.  Unlike the other models,  \scar \ and \vscar \ do not reach close to $100\%$ accuracy rate. We hypothesize that this is due to the fact that, unlike digits, not every combination of letters forms a valid label, so many words have an edit distance of multiple characters to get to the closest different label. In these experiments, \pa \ obtains the highest success rate.

\begin{figure}
	\centering
	\begin{minipage}[t]{.45\textwidth}
		\centering
		\includegraphics[width=.95\linewidth]{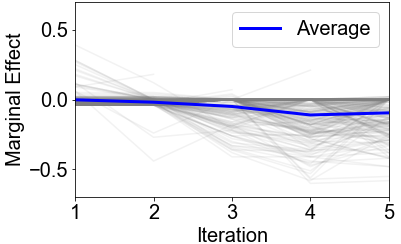}
		\caption{The gain from each pixel for the five iterations it took to successfully attack the word ``idle'' on Tesseract.}
		\label{fig:temporal}
	\end{minipage}%
	\hspace{1cm}
	\begin{minipage}[t]{.45\textwidth}
		\centering
		\raisebox{.1cm}{\includegraphics[width=.95\linewidth]{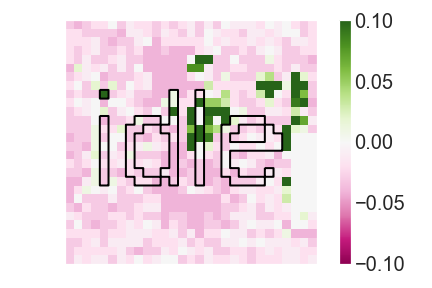}}
		\caption{Heatmap of the gains from flipping a pixel on the word ``idle" with Tesseract.}
		\label{fig:spatial}
	\end{minipage}
\end{figure}

\paragraph{Spatial and temporal correlations.} As discussed in Section~\ref{sec:alg}, \scar \ exploits spatial and temporal correlations to optimize the number of queries needed. As an example, we consider \scar \ attacking Tesseract on the word ``idle''.

In Figure~\ref{fig:temporal} we plot a separate line for each pixel $p$ and the corresponding decrease in confidence from flipping that pixel at each iteration. We first note the pixels with the smallest gains at some iteration are often among the pixels with the smallest gains in the next iteration, which indicates temporal correlations.  Most of the gains are negative, which implies that, surprisingly, for most pixels, flipping that pixel  \emph{increases} the confidence of the true label. Thus, randomly choosing which pixel to flip, as in \simba, is ineffective.

Figure~\ref{fig:spatial} again shows the gain from flipping each pixel, but this time as a heatmap for the gains at the first iteration. We note that most pixels with a large gain have at least one neighboring pixel that also has a large gain. This heatmap illustrates that  first querying the neighboring pixels of the previous pixel flipped is an effective technique to reduce the number of queries needed to find a high gain pixel.

\subsection{Check processing systems}
\label{sec:checkresults}

\begin{figure*}
	\centering
	\includegraphics[width=1\linewidth]{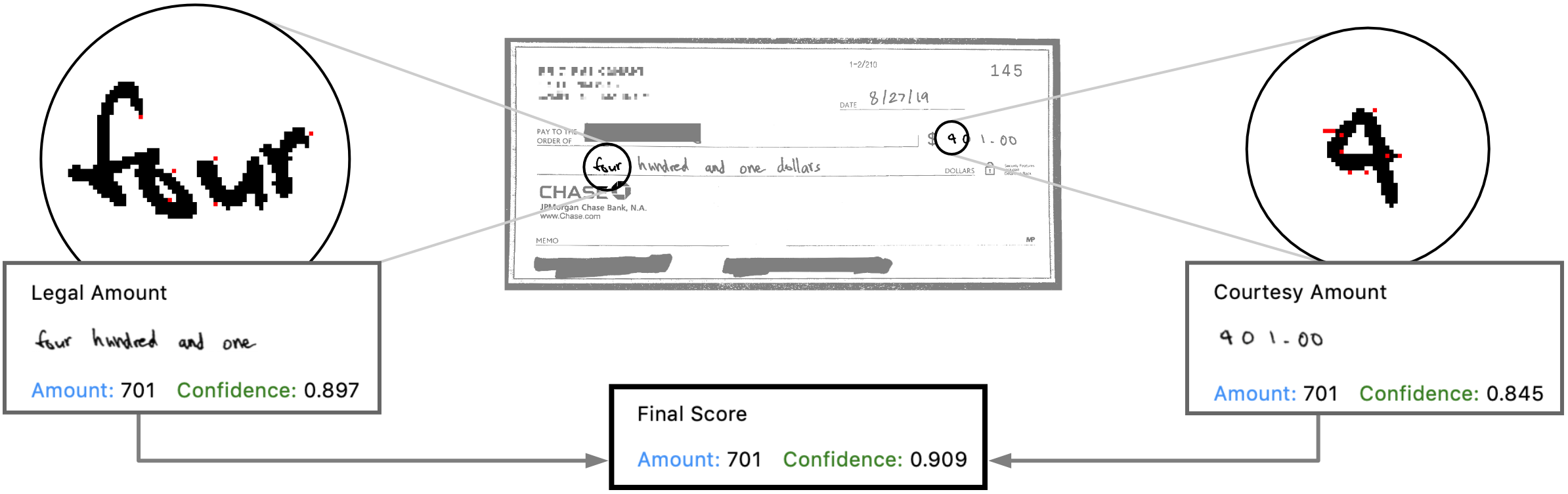}
	\caption{An example of a check for $\$401$ attacked by Algorithm~\ref{alg:check} that is misclassified with high confidence as $\$701$ by a check processing system used by US banks.}
	\label{fig:full_check} 
\end{figure*}

We licensed software from providers  of check processing systems to major US banks and applied the attack described in Algorithm~\ref{alg:check}. This software includes the prediction confidence as part of their output. Naturally,  access to these systems is limited and the cost per query is significant. We confirm the findings from the previous experiments  that \scar, which is used as a subroutine by Algorithm~\ref{alg:check}, is effective in query-limited settings and showcase the vulnerability of OCR systems used in the industry. Check fraud is a major concern for US banks; it caused over $\$1.3$ billion in losses in 2018~\cite{checkreport}.

We obtained a $17.1\%$ success rate ($19$ out of $111$ checks) when attacking check processing systems used by banks for mobile check deposits. As previously mentioned, a check is successfully attacked when both amounts on the check are misclassified as the same wrong amount (see Figure~\ref{fig:full_check}). Since check fraud occurs at large scale, we believe that this vulnerability raises serious concerns.\footnote{Regarding physical realizability: instead of printing  an adversarial check in high resolution, an attacker can redirect the camera input of a mobile phone to arbitrary image files, which avoids printing and taking a picture of an adversarial check.   This hacking of the camera input is  easy to perform on Android.}

\begin{table}[H]
	\centering
	\begin{tabular}{c c c}
		\textbf{Classifier} & \textbf{Queries} & \textbf{$L_{0}$ distance}\\ \hline
		CAR ($F_C$) &
		$1615$ & $11.77$ \\
		\hline
		LAR ($F_L$) &
		$8757$ &  $14.85$ \\
	\end{tabular}
	\label{t:table}
	\caption{Average number of queries and $L_0$ distance on the CAR and LAR classifiers for the high confidence successful attacks.}
\end{table}

We say that a check is misclassified with high confidence if  the amounts written in number and words are each classified with confidence at least $50\%$ for the wrong label. We obtained high confidence misclassification for $76.5\%$ of the checks successfully attacked.   In Figure~\ref{fig:full_check}, we show the output of a check for $\$401$  that has both amounts classified as $701$ with confidence at least $80\%$. On average, over the checks for which we obtained high confidence misclassification, Algorithm~\ref{alg:check} flipped $11.77$  and $14.85$ pixels and made $1615$ and $8757$ queries for the amounts in numbers and words respectively. The checks are  high resolution, with widths of size $1000$. Additional examples of checks misclassified with high confidence can be found in the appendix.

\newpage

\bibliography{biblio}

\begin{thebibliography}{30}
\providecommand{\natexlab}[1]{#1}
\providecommand{\url}[1]{\texttt{#1}}
\expandafter\ifx\csname urlstyle\endcsname\relax
  \providecommand{\doi}[1]{doi: #1}\else
  \providecommand{\doi}{doi: \begingroup \urlstyle{rm}\Url}\fi

\bibitem[Al-Dujaili and O'Reilly(2019)]{al2019there}
Abdullah Al-Dujaili and Una-May O'Reilly.
\newblock There are no bit parts for sign bits in black-box attacks.
\newblock \emph{arXiv preprint arXiv:1902.06894}, 2019.

\bibitem[{American Bankers Association}(2020)]{checkreport}
{American Bankers Association}.
\newblock Deposit account fraud survey.
\newblock 2020.
\newblock URL
  \url{https://www.aba.com/news-research/research-analysis/deposit-account-fraud-survey-report}.

\bibitem[Carlini and Wagner(2017)]{lzerocarlini}
Nicholas Carlini and David Wagner.
\newblock Towards evaluating the robustness of neural networks.
\newblock In \emph{2017 IEEE symposium on security and privacy (SP)}, pages
  39--57. IEEE, 2017.

\bibitem[Chen et~al.(2017)Chen, Zhang, Sharma, Yi, and Hsieh]{bbox1}
Pin-Yu Chen, Huan Zhang, Yash Sharma, Jinfeng Yi, and Cho-Jui Hsieh.
\newblock Zoo: Zeroth order optimization based black-box attacks to deep neural
  networks without training substitute models.
\newblock In \emph{Proceedings of the 10th ACM Workshop on Artificial
  Intelligence and Security}, pages 15--26, 2017.

\bibitem[Cohen et~al.(2017)Cohen, Afshar, Tapson, and Van~Schaik]{emnist}
Gregory Cohen, Saeed Afshar, Jonathan Tapson, and Andre Van~Schaik.
\newblock {EMNIST}: Extending {MNIST} to handwritten letters.
\newblock In \emph{2017 International Joint Conference on Neural Networks
  (IJCNN)}, pages 2921--2926. IEEE, 2017.

\bibitem[Diem et~al.(2014)Diem, Fiel, Kleber, Sablatnig, Saavedra, Contreras,
  Barrios, and Oliveira]{diem2014icfhr}
Markus Diem, Stefan Fiel, Florian Kleber, Robert Sablatnig, Jose~M Saavedra,
  David Contreras, Juan~Manuel Barrios, and Luiz~S Oliveira.
\newblock {ICFHR} 2014 competition on handwritten digit string recognition in
  challenging datasets ({HDSRC} 2014).
\newblock In \emph{2014 14th International Conference on Frontiers in
  Handwriting Recognition}, pages 779--784. IEEE, 2014.

\bibitem[Ding et~al.(2019)Ding, Lui, Jin, Wang, and Huang]{ding2019sensitivity}
Gavin~Weiguang Ding, Kry Yik~Chau Lui, Xiaomeng Jin, Luyu Wang, and Ruitong
  Huang.
\newblock On the sensitivity of adversarial robustness to input data
  distributions.
\newblock In \emph{ICLR (Poster)}, 2019.

\bibitem[Goodfellow et~al.(2014)Goodfellow, Shlens, and Szegedy]{fgsm}
Ian~J Goodfellow, Jonathon Shlens, and Christian Szegedy.
\newblock Explaining and harnessing adversarial examples.
\newblock \emph{arXiv preprint arXiv:1412.6572}, 2014.

\bibitem[Guo et~al.(2019)Guo, Gardner, You, Wilson, and Weinberger]{simba}
Chuan Guo, Jacob~R Gardner, Yurong You, Andrew~Gordon Wilson, and Kilian~Q
  Weinberger.
\newblock Simple black-box adversarial attacks.
\newblock \emph{arXiv preprint arXiv:1905.07121}, 2019.

\bibitem[Gupta et~al.(2007)Gupta, Jacobson, and Garcia]{document_extraction}
Maya~R Gupta, Nathaniel~P Jacobson, and Eric~K Garcia.
\newblock {OCR} binarization and image pre-processing for searching historical
  documents.
\newblock \emph{Pattern Recognition}, 40\penalty0 (2):\penalty0 389--397, 2007.

\bibitem[Ilyas et~al.(2018{\natexlab{a}})Ilyas, Engstrom, Athalye, and
  Lin]{bboxmadry}
Andrew Ilyas, Logan Engstrom, Anish Athalye, and Jessy Lin.
\newblock Black-box adversarial attacks with limited queries and information.
\newblock \emph{arXiv preprint arXiv:1804.08598}, 2018{\natexlab{a}}.

\bibitem[Ilyas et~al.(2018{\natexlab{b}})Ilyas, Engstrom, and
  Madry]{bboxmadry2}
Andrew Ilyas, Logan Engstrom, and Aleksander Madry.
\newblock Prior convictions: Black-box adversarial attacks with bandits and
  priors.
\newblock \emph{arXiv preprint arXiv:1807.07978}, 2018{\natexlab{b}}.

\bibitem[Jayadevan et~al.(2012)Jayadevan, Kolhe, Patil, and
  Pal]{check_binarization}
R~Jayadevan, Satish~R Kolhe, Pradeep~M Patil, and Umapada Pal.
\newblock Automatic processing of handwritten bank cheque images: a survey.
\newblock \emph{International Journal on Document Analysis and Recognition
  (IJDAR)}, 15\penalty0 (4):\penalty0 267--296, 2012.

\bibitem[Kurakin et~al.(2016)Kurakin, Goodfellow, and Bengio]{bim}
Alexey Kurakin, Ian Goodfellow, and Samy Bengio.
\newblock Adversarial examples in the physical world.
\newblock \emph{arXiv preprint arXiv:1607.02533}, 2016.

\bibitem[LeCun et~al.(1998)LeCun, Bottou, Bengio, and
  Haffner]{lecun1998gradient}
Yann LeCun, L{\'e}on Bottou, Yoshua Bengio, and Patrick Haffner.
\newblock Gradient-based learning applied to document recognition.
\newblock \emph{Proceedings of the IEEE}, 86\penalty0 (11):\penalty0
  2278--2324, 1998.

\bibitem[LeCun et~al.(2010)LeCun, Cortes, and Burges]{mnist}
Yann LeCun, Corinna Cortes, and CJ~Burges.
\newblock {MNIST} handwritten digit database.
\newblock 2010.

\bibitem[Lei et~al.(2018)Lei, Wu, Chen, Dimakis, Dhillon, and
  Witbrock]{combi_dimakis}
Qi~Lei, Lingfei Wu, Pin-Yu Chen, Alexandros~G Dimakis, Inderjit~S Dhillon, and
  Michael Witbrock.
\newblock Discrete adversarial attacks and submodular optimization with
  applications to text classification.
\newblock \emph{arXiv preprint arXiv:1812.00151}, 2018.

\bibitem[Li et~al.(2019)Li, Li, Wang, Zhang, and Gong]{nattack}
Yandong Li, Lijun Li, Liqiang Wang, Tong Zhang, and Boqing Gong.
\newblock Nattack: Learning the distributions of adversarial examples for an
  improved black-box attack on deep neural networks.
\newblock \emph{arXiv preprint arXiv:1905.00441}, 2019.

\bibitem[Madry et~al.(2017)Madry, Makelov, Schmidt, Tsipras, and Vladu]{pgd}
Aleksander Madry, Aleksandar Makelov, Ludwig Schmidt, Dimitris Tsipras, and
  Adrian Vladu.
\newblock Towards deep learning models resistant to adversarial attacks.
\newblock \emph{arXiv preprint arXiv:1706.06083}, 2017.

\bibitem[Moon et~al.(2019)Moon, An, and Song]{combi-colored}
Seungyong Moon, Gaon An, and Hyun~Oh Song.
\newblock Parsimonious black-box adversarial attacks via efficient
  combinatorial optimization.
\newblock \emph{arXiv preprint arXiv:1905.06635}, 2019.

\bibitem[Moosavi-Dezfooli et~al.(2016)Moosavi-Dezfooli, Fawzi, and
  Frossard]{deepfool}
Seyed-Mohsen Moosavi-Dezfooli, Alhussein Fawzi, and Pascal Frossard.
\newblock Deepfool: a simple and accurate method to fool deep neural networks.
\newblock In \emph{Proceedings of the IEEE conference on computer vision and
  pattern recognition}, pages 2574--2582, 2016.

\bibitem[Mor and Wolf(2018)]{mor2018confidence}
Noam Mor and Lior Wolf.
\newblock Confidence prediction for lexicon-free {OCR}.
\newblock In \emph{2018 IEEE Winter Conference on Applications of Computer
  Vision (WACV)}, pages 218--225. IEEE, 2018.

\bibitem[Papernot et~al.(2016)Papernot, McDaniel, Jha, Fredrikson, Celik, and
  Swami]{lzeropapernot}
Nicolas Papernot, Patrick McDaniel, Somesh Jha, Matt Fredrikson, Z~Berkay
  Celik, and Ananthram Swami.
\newblock The limitations of deep learning in adversarial settings.
\newblock In \emph{2016 IEEE European symposium on security and privacy
  (EuroS\&P)}, pages 372--387. IEEE, 2016.

\bibitem[Rauber et~al.(2017)Rauber, Brendel, and Bethge]{rauber2017foolbox}
Jonas Rauber, Wieland Brendel, and Matthias Bethge.
\newblock Foolbox: A python toolbox to benchmark the robustness of machine
  learning models.
\newblock \emph{arXiv preprint arXiv:1707.04131}, 2017.

\bibitem[Schmidt et~al.(2018)Schmidt, Santurkar, Tsipras, Talwar, and
  Madry]{madry_binarization}
Ludwig Schmidt, Shibani Santurkar, Dimitris Tsipras, Kunal Talwar, and
  Aleksander Madry.
\newblock Adversarially robust generalization requires more data.
\newblock In \emph{Advances in Neural Information Processing Systems}, pages
  5014--5026, 2018.

\bibitem[Schott et~al.(2018)Schott, Rauber, Bethge, and Brendel]{pointwise}
Lukas Schott, Jonas Rauber, Matthias Bethge, and Wieland Brendel.
\newblock Towards the first adversarially robust neural network model on
  {MNIST}.
\newblock \emph{arXiv preprint arXiv:1805.09190}, 2018.

\bibitem[Smith(2007)]{smith2007overview}
Ray Smith.
\newblock An overview of the {Tesseract} {OCR} engine.
\newblock In \emph{Ninth International Conference on Document Analysis and
  Recognition (ICDAR 2007)}, volume~2, pages 629--633. IEEE, 2007.

\bibitem[Song and Shmatikov(2018)]{ocr}
Congzheng Song and Vitaly Shmatikov.
\newblock Fooling {OCR} systems with adversarial text images.
\newblock \emph{arXiv preprint arXiv:1802.05385}, 2018.

\bibitem[Szegedy et~al.(2013)Szegedy, Zaremba, Sutskever, Bruna, Erhan,
  Goodfellow, and Fergus]{szegedy13}
Christian Szegedy, Wojciech Zaremba, Ilya Sutskever, Joan Bruna, Dumitru Erhan,
  Ian Goodfellow, and Rob Fergus.
\newblock Intriguing properties of neural networks.
\newblock \emph{arXiv preprint arXiv:1312.6199}, 2013.

\bibitem[Tu et~al.(2019)Tu, Ting, Chen, Liu, Zhang, Yi, Hsieh, and
  Cheng]{bbox2}
Chun-Chen Tu, Paishun Ting, Pin-Yu Chen, Sijia Liu, Huan Zhang, Jinfeng Yi,
  Cho-Jui Hsieh, and Shin-Ming Cheng.
\newblock Autozoom: Autoencoder-based zeroth order optimization method for
  attacking black-box neural networks.
\newblock In \emph{Proceedings of the AAAI Conference on Artificial
  Intelligence}, volume~33, pages 742--749, 2019.

\end{thebibliography}
\bibliographystyle{plainnat}

\newpage

\section*{Appendix}

\appendix

\section{Missing analysis from Section~\ref{sec:robust}}

\thmone*
\begin{proof} In this proof, we assume that binary images have pixel values in  $\{-1, 1\}$ instead of $\{0,1\}$. We consider a linear classifier $F_{\bw^{\star}_1, \ldots, \bw^{\star}_m}$ such that the predicted label $y$ of a binary image $\bx$ is $y = \argmax_i \bx^\intercal \bw^{\star}_i$.
	
	We wish to show  the existence of weight vectors $\bw^{\star}_1, \ldots, \bw^{\star_m}$ that all have large pairwise $L_0$ distance. This is closely related to error-correction codes in coding theory which, in order to detect and reconstruct a noisy codes, also aims to construct binary codes with large pairwise distance.
	
	We do this using the probabilistic method. Consider $m$ binary weight vectors $\bw_1, \ldots, \bw_m$ chosen uniformly at random, and independently, among all $d$-dimensional binary vectors  $\bw \in \{-1,1\}^d$.  By the Chernoff bound, for all $i, j \in [m]$, we have that for $0 < \delta < 1$,
	$$\Pr\left[ \|\bw_i - \bw_j\|_0 \leq (1 - \delta) d / 2\right] \leq e^{-\delta^2 d /4}.$$
	There are $\binom{m}{2} < m^2$ pairs of images $(i,j)$. By a union bound and with $\delta =   \sqrt{8\log m /d}$, we get that 
	$$\Pr\left[ \|\bw_i - \bw_j\|_0 > d/2 - \sqrt{2d \log m} : \text{for all } i, j \in [m], i \neq j \right] > 1 - m^2 e^{-\delta^2 d / 4} > 0.$$ Thus, by the probabilistic method, there exists $\bw^{\star}_1, \ldots, \bw^{\star}_m$ such that $ \|\bw^{\star}_i - \bw^{\star}_j\|_0 >  d/2 - \sqrt{2d \log m} $ for all $i, j \in [m]$.
	
	It remains to show that the linear classifier $F_{\bw^{\star}_1, \ldots, \bw^{\star_m}}$ satisfies the condition of the theorem statement.
	For class $i$, consider the binary image $\bx_i = \bw^{\star}_i$. Note that for binary images $\bx \in \{-1, 1\}^d$, we have $\bx^\intercal  \bw^{\star}_i  = d - 2 \|\bx - \bw^{\star}_i\|_0$.  Thus,  $\bx^\intercal_i  \bw^{\star}_i = d$ and $\argmax_{j \neq i} \bx^\intercal_i  \bw^{\star}_j < 2\sqrt{2d \log m}$, and we get  $\bx_i^\intercal  \bw^{\star}_i  - \argmax_{j \neq i} \bx_i^\intercal  \bw^{\star}_j > d - 2\sqrt{2d \log m}$. Each pixel change reduces this difference by at most $4$. Thus, for all $\bx'$ such that  $\|\bx_i - \bx'\|_0   \leq (d - 2\sqrt{2d \log m})/4 = d/4 - \sqrt{2d \log m}/2$, we have $\bx'^\intercal  \bw^{\star}_i  - \argmax_{j \neq i} \bx'^\intercal  \bw^{\star}_j > 0$ and the predicted label of $\bx'$ is $i$.
\end{proof}

\thmtwo*
\begin{proof}
	Consider the following linear classifier
	\[F(\bx) = \begin{cases}0 & \text{if $\vec{1}^{T}\vec{x} -  x_0 /2 < \frac{d}{2}$} \\ 1 & \text{otherwise}\end{cases}.\]
	Informally, this is a classifier which assigns label $0$ if $\Vert \bx \Vert_{0} < d/2$ and label $1$ if $\Vert \bx \Vert_{0} > d/2$. The classifier tiebreaks the $\Vert \bx \Vert_{0} = d/2$ case depending on whether or not the first position in $\bx$ is a $1$ or a $0$. Notice that this classifier assigns exactly half the space the label $0$, and the other half the label $1$.

	Consider class $0$ and let $ \U(0)$ be the uniform distribution over all $\bx$ in class $0$. We have $$\Pr_{\bx \in \U(0)} [\Vert \bx\Vert_{0} = s] = \frac{1}{2^{d}}\binom{d}{s}$$ when $s < d/2$ and $\Pr_{\bx \in \U(0)} [\Vert \bx\Vert_{0} = s] = \frac{1}{2^{d+1}}\binom{d}{s}$ when $s = d/2$. The binomial coefficient $\binom{d}{s}$ is maximized when $s = d/2 $.
	For all $d \in \mathbb{Z}^{+}$, Stirling's approximation gives lower and upper bounds of $\sqrt{2\pi} d^{d+\frac{1}{2}}e^{-d} \leq d! \leq d^{d + \frac{1}{2}}e^{-d+1}$. Since $d$ is even, we get  $$ \binom{d}{d/2} = \frac{d!}{(\frac{d}{2}!)^{2}} \leq \frac{e2^{d}}{\pi\sqrt{d}}.$$  Therefore, we have that for all $s$,
	\[\Pr_{\bx \in \U(0)} [\Vert \bx\Vert_{0} = s] \leq \frac{1}{2^{d}}\binom{d}{ d/2 } \leq \frac{e2^{d}}{\pi\sqrt{d}},\]
	which implies 
	\[\Pr_{\bx \in \U(0)} \left[\left|\Vert \bx\Vert_{0} - d/2\right| \geq \frac{\pi\sqrt{d}}{4e}\right] \geq 1 - \frac{2\pi\sqrt{d}}{4e}\cdot\frac{e}{\pi\sqrt{d}} = \frac{1}{2}.\]
	The same argument follows similarly for members of class $1$. Therefore, for either class, at least half of the images $\vec{x}$ of that class are such that $\left|\Vert \bx\Vert_{0} - d/2\right| \geq \frac{\pi\sqrt{d}}{4e} \geq \frac{\sqrt{d}}{4}$. These images require at least $\frac{\sqrt{d}}{4}$ pixel flips in order to change the predicted label according to $F$, and we obtain the bound in the theorem statement.
\end{proof}

\section{Additional Description of Datasets and Models}

\subsection{Digit and character recognition systems}

\paragraph{The datasets.} We preprocess the EMNIST letter dataset to only include lowercase letters, since an uppercase letter which is misclassified as the corresponding lowercase letter does not change the semantic meaning of the overall word. We randomly select 10 correctly-classified samples  from each class in MNIST and EMNIST lowercase letters to form two datasets to attack.

\paragraph{Models.}  We consider the following five models, trained in the same manner for the MNIST and EMNIST  datasets. For each model, we also list their Top-1 accuracies on MNIST and EMNIST.

\begin{itemize}
	\item \textbf{LogReg:} We create a logistic regression model by flattening the input and follow this with a fully connected layer with softmax activation. (MNIST: 91.87\%\ /\ EMNIST: 80.87\%)
	\item \textbf{MLP2:} We create a 2-layer MLP by flattening the input, followed by two sets of fully connected layers of size 512 with ReLU activation and dropout rate $0.2$. We then add a fully connected layer with softmax activation. (MNIST: 98.01\%\ /\ EMNIST: 93.46\%)
	\item \textbf{CNN:} We use two convolutional layers of 32 and 64 filters of size $3 \times 3$, each with ReLU activation. The latter layer is followed by a $2 \times 2$ Max Pooling layer with dropout rate $0.25$. (MNIST: 99.02\%\ /\ EMNIST: 95.04\%)
	
	This output is flattened and followed by a fully connected layer of size 128 with ReLU activation and dropout rate $0.5$. We then add a fully connected layer with softmax activation.
	\item \textbf{LeNet 5:} We use the same architecture as in \cite{lecun1998gradient}. (MNIST: 99.01\%\ /\ EMNIST: 94.33\%)
	\item \textbf{SVM:} We use the \texttt{sklearn} implementation with default parameters. (MNIST: 94.11\%\ /\ EMNIST: 87.53\%)
\end{itemize}

Except for the SVM, we train each model for 50 epochs with batch size 128, using the Adam optimizer with a learning rate of $10^{-3}$. The experimental results for CNN on MNIST and LeNet5 on EMNIST are shown in Section~5.

\subsection{LSTM on handwriten numbers}

\paragraph{The dataset.}
We train an OCR model on the ORAND-CAR-A dataset, part of the HDSRC 2014 competition on handwritten strings \cite{diem2014icfhr}. This dataset consists of 5793 images from real bank checks taken from a Uruguayan bank. The characters in these images consist of numeric characters (0-9) and each image contains between 2 and 8 characters. These images also contain some background noise due the real nature of the images. We observe the train/test split given in the initial competition, meaning that we train our model on 2009 images and attack only a randomly selected subset from the test set (another 3784 images). The images as presented in the competition were colored, but we binarize them in a similar preprocessing step as done for MNIST/EMNIST datasets.

\paragraph{The LSTM model.}
We implement the OCR model described in \cite{mor2018confidence}, which consists of a convolutional layer, followed by a 3-layer deep bidirectional LSTM, and optimizes for CTC loss. CTC decoding was done using a beam search of width 100. The model was trained with the Adam optimizer using a learning rate of $10^{-4}$, and was trained for 50 epochs. The trained model achieves a precision score of $.857$ on the test set of ORAND-CAR-A, which would have achieved first place in that competition.

\subsection{Tesseract on printed words}

\paragraph{The model.} We use Tesseract version 4.1.1 trained for the English language. Tesseract 4 is based on an LSTM model (see \cite{ocr} for a detailed description of the architecture of Tesseract's model). 

\paragraph{The dataset.} We attack images of a single  printed English word. Tesseract supports a large number of languages, and we use the version of Tesseract trained for the English language. We picked words of length four in the English dictionary. We then rendered these words in black over a white background using the Arial font in size 15. We added 10 white pixels for padding on each side of the word.  The accuracy rate over $1000$ such images of English words of length four chosen at random is $0.965$ and the average confidence among words correctly classified is $0.906$. Among the words correctly classified by Tesseract, we selected $100$ at random to attack.

For some attacked images with a lot of noise, Tesseract does not recognize any word and rejects the input. Since the goal of these attacks is to misclassify images as words with a different meaning, we only consider an attack to be successful if the adversarial image produced is classified as a word in the English dictionary. For example, consider an attacked image of the word ``one". If Tesseract does not recognize any word in this image, or recognizes ``oe" or ``:one", we do not count this image as a successful attack.

We restricted the attacks to pixels that were at distance at most three of the box around the word.  Since our algorithm only considers boundary pixels, this restriction avoids  giving an unfair advantage to our algorithm in terms of total number of queries. In some cases, especially images with a lot of noise, Tesseract does not recognize any word and rejects the input. Since the goal of these attacks is to misclassify images as words with a different meaning than the true word, we only consider an attack to be successfull if the adversarial image produced is classified as a word in the English dictionary. For example, consider an image with the text ``one". If Tesseract does not recognize any word in this image, or recognizes ``oe" or ``:one", we do not count this image as a successful attack.  

\begin{figure*}
	\centering
	\includegraphics[trim=.5cm 0 0 0, width=.225\linewidth]{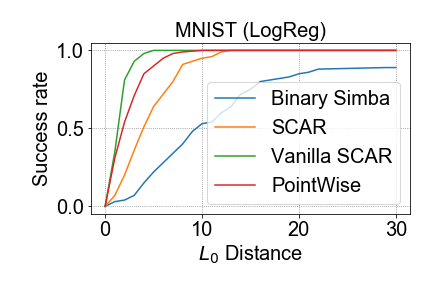}
	\hspace{.2cm}
	\includegraphics[ width=.225\linewidth]{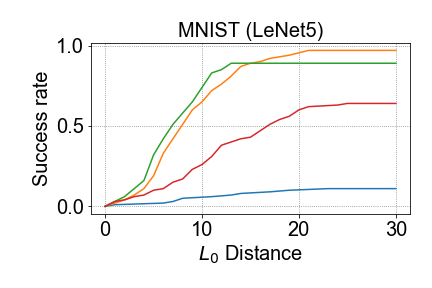}
	\hspace{.2cm}
	\includegraphics[width=.225\linewidth]{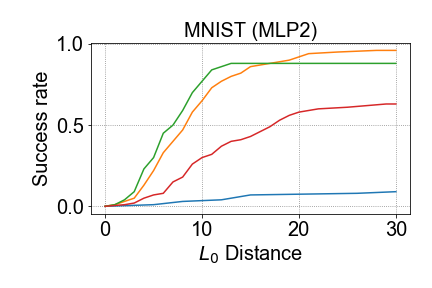}
	\hspace{.2cm}
	\includegraphics[width=.225\linewidth]{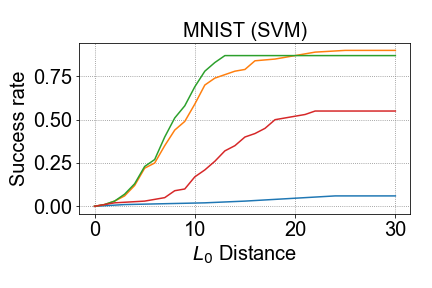}
	\includegraphics[width=.8\linewidth]{blank.png}
	\includegraphics[width=.225\linewidth]{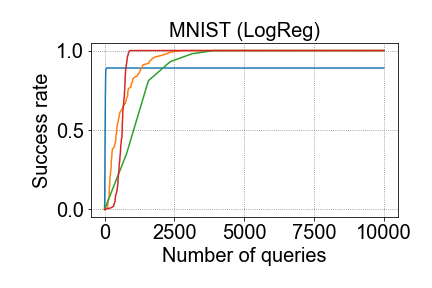}
	\hspace{.2cm}
	\includegraphics[width=.235\linewidth]{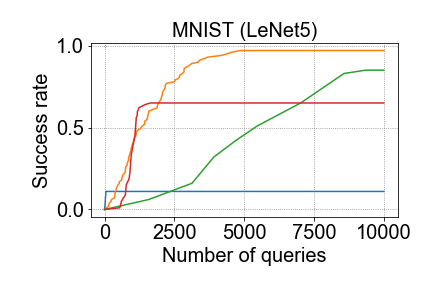}
	\hspace{.2cm}
	\includegraphics[width=.225\linewidth]{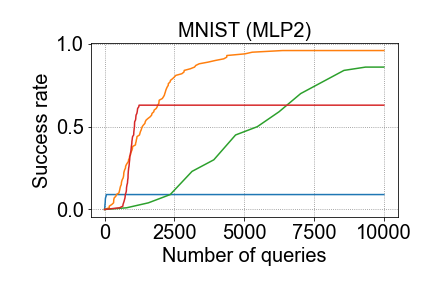}
	\hspace{.2cm}
	\includegraphics[width=.225\linewidth]{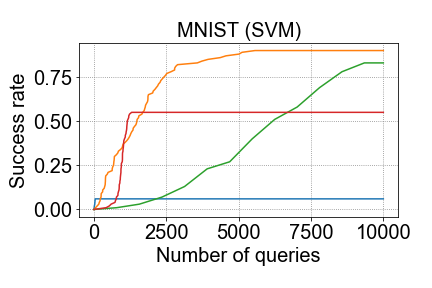}
	\caption{Success rate by $L_0$ distance and by number of queries for four different models on MNIST.}
	\label{fig:plotsMNIST}
\end{figure*}

\begin{figure*}
	\centering       
	\includegraphics[trim=.5cm 0 0 0, width=.225\linewidth]{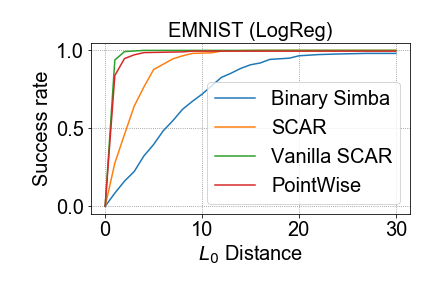}
	\hspace{.2cm}
	\includegraphics[width=.225\linewidth]{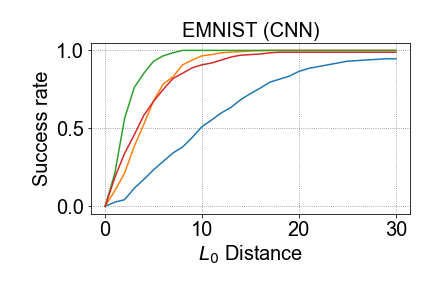}
	\hspace{.2cm}
	\includegraphics[width=.225\linewidth]{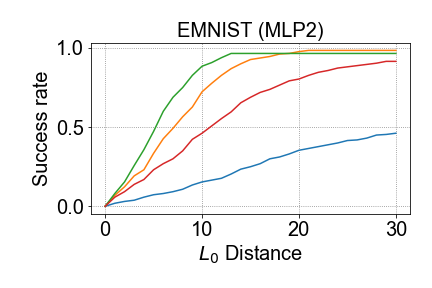}
	\hspace{.2cm}
	\includegraphics[width=.225\linewidth]{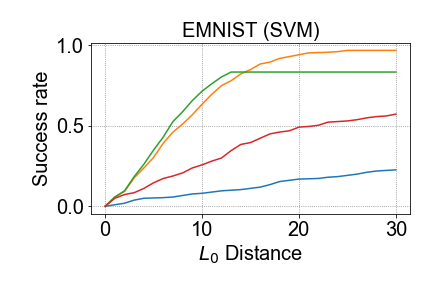}
	\includegraphics[width=.8\linewidth]{blank.png}
	\includegraphics[width=.225\linewidth]{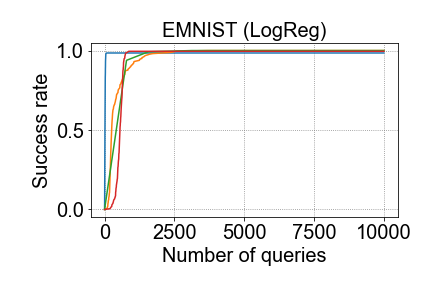}
	\hspace{.2cm}
	\includegraphics[width=.235\linewidth]{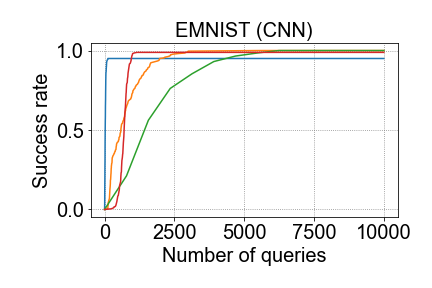}
	\hspace{.2cm}
	\includegraphics[width=.225\linewidth]{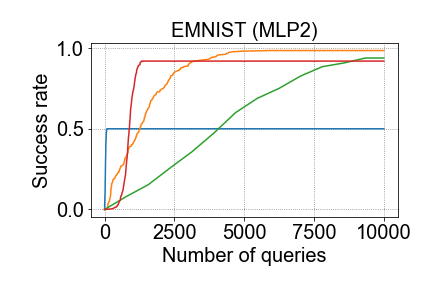}
	\hspace{.2cm}
	\includegraphics[width=.225\linewidth]{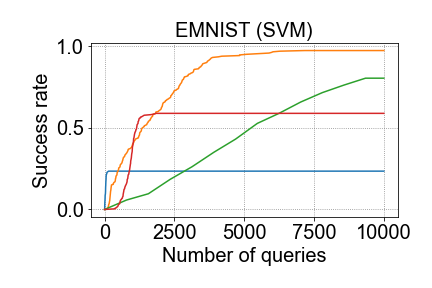}
	\caption{Success rate by $L_0$ distance and by number of queries for four different models on EMNIST.}
	\label{fig:plotsEMNIST}
\end{figure*}

\begin{figure}[H]
	\centering
	\includegraphics[width=.19\linewidth]{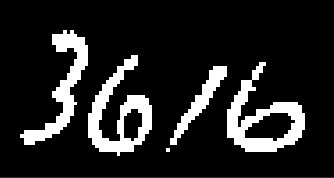}
	\includegraphics[width=.19\linewidth]{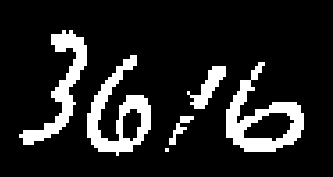}
	\includegraphics[width=.19\linewidth]{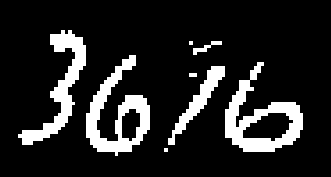}
	\includegraphics[width=.19\linewidth]{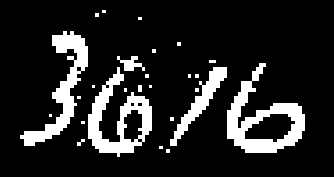}
	\includegraphics[width=.19\linewidth]{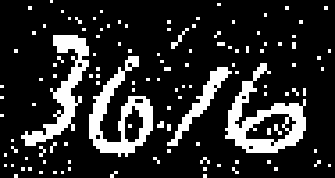}
	\caption{Examples of attacks on the LSTM for handwritten numbers. The images correspond to, from left to right, the original image, the outputs of \scar, \vscar, \pa, and \simba.}
	\label{fig:3616}
\end{figure}

\section{Additional Experimental Results}

In Figure~\ref{fig:plotsMNIST} and Figure~\ref{fig:plotsEMNIST}, we provide additional experimental results on the MNIST and EMNIST datasets. In Figure~\ref{fig:3616}, we give additional examples of attacks on the LSTM model for handwritten number. In Table~1, we list the $100$ English words of length $4$ we attacked together with the word label of the image resulting from running \scar.

\begin{table*}[h]
	\centering
	\begin{tabular}{| c | c | c | c | c | c | c | c | c |}
		\hline
		\textbf{Original} & \textbf{Label} & \hspace{1cm} & \textbf{Original} & \textbf{Label} & \hspace{1cm} & \textbf{Original} & \textbf{Label}  \\
		\textbf{word} & \textbf{from \scar} & \hspace{1cm} & \textbf{word} & \textbf{from \scar} & \hspace{1cm} &\textbf{word} & \textbf{from \scar} \\\hline \hline
		down& dower   & &
		race	 & rate & &
		punt	  & pant  \\ \hline
		fads  &	fats & &
		nosy&	rosy & &
		mans&	mans \\ \hline
		pipe&	pie & &
		serf&	set & &
		cram&	ram \\ \hline
		soft	& soft & &
		dare&	dare & &
		cape&	tape \\ \hline
		pure&	pure & &
		hood	&hoot & &
		bide&	hide \\ \hline
		zoom&	zoom & &
		yarn&	yam & &
		full&	fall \\ \hline
		lone&	tone & &
		gorp&	gore & &
		lags&	fags \\ \hline
		fuck&	fucks & &
		fate&	ate & &
		dolt&	dot \\ \hline
		fist&	fist & &
		mags&	mays & &
		mods	& mots \\ \hline
		went	& weal & &
		oust &	bust & &
		game&	game \\ \hline
		omen&	men & &
		rage	& rage & &
		taco	& taco \\ \hline
		idle	& die & &
		moth	 & math & &
		ecol	& col \\ \hline
		yeah	& yeah & &
		woad &	woad & &
		deaf	& deaf \\ \hline
		feed	& feet & &
		aged	& ed & &
		vary	& vary \\ \hline
		nuns	& runs & &
		dray	& ray & &
		tell	& tel \\ \hline
		educ	& educ & &
		ency	& ency & &
		avow &	vow \\ \hline
		gush	& gust & &
		pres	& press & &
		wits &	wits \\ \hline
		news	 & news & &
		deep	& sleep & &
		weep &	ween \\ \hline
		swim	 & swim & &
		bldg	& bid & &
		vile	& vie  \\ \hline
		hays	& nays & &
		warp	& war & &
		sets	& nets \\ \hline
		tube	& lube & &
		lost	& lo & &
		smut	& snout \\ \hline
		lure	& hare & &
		sqrt	& sat & &
		mies	& miles \\ \hline
		romp	 & romp & &
		okay	& okay & &
		boot	& hoot \\ \hline
		comp &	camp & &
		kept	& sept & &
		yipe	& vie \\ \hline
		pith &	pithy & &
		herb	& herbs & &
		hail	& fail \\ \hline
		ploy	& pro & &
		show &	how & &
		saga	& gaga \\ \hline
		toot	& foot & &
		hick	& nick & &
		drat	& rat \\ \hline
		boll	& boil & &
		tout	& foul & &
		limo	& lino \\ \hline
		elev	& ale & &
		blur	& bur & &
		idem	& idler \\ \hline
		dank	& dank & &
		biog	& dog & &
		twin	& twins \\ \hline
		gild	& ail & &
		lain	& fain & &
		slip	& sip \\ \hline
		waxy	 & waxy & &
		gens	& gents & &
		yeti	& yet \\ \hline
		test	& fest & &
		mega	& mega & &
		loge	& toge \\ \hline
		pups	& pups & & & &  \\ \hline
	\end{tabular}
	\label{table}
	\caption{The $100$ English words of length $4$ we attacked together with the word label of the image resulting from running \scar.}
\end{table*}

Finally, in Figure~6, we show additional examples of our attacks on check processing systems.

\begin{figure}[h]
	\centering
	\includegraphics[width=.13\linewidth]{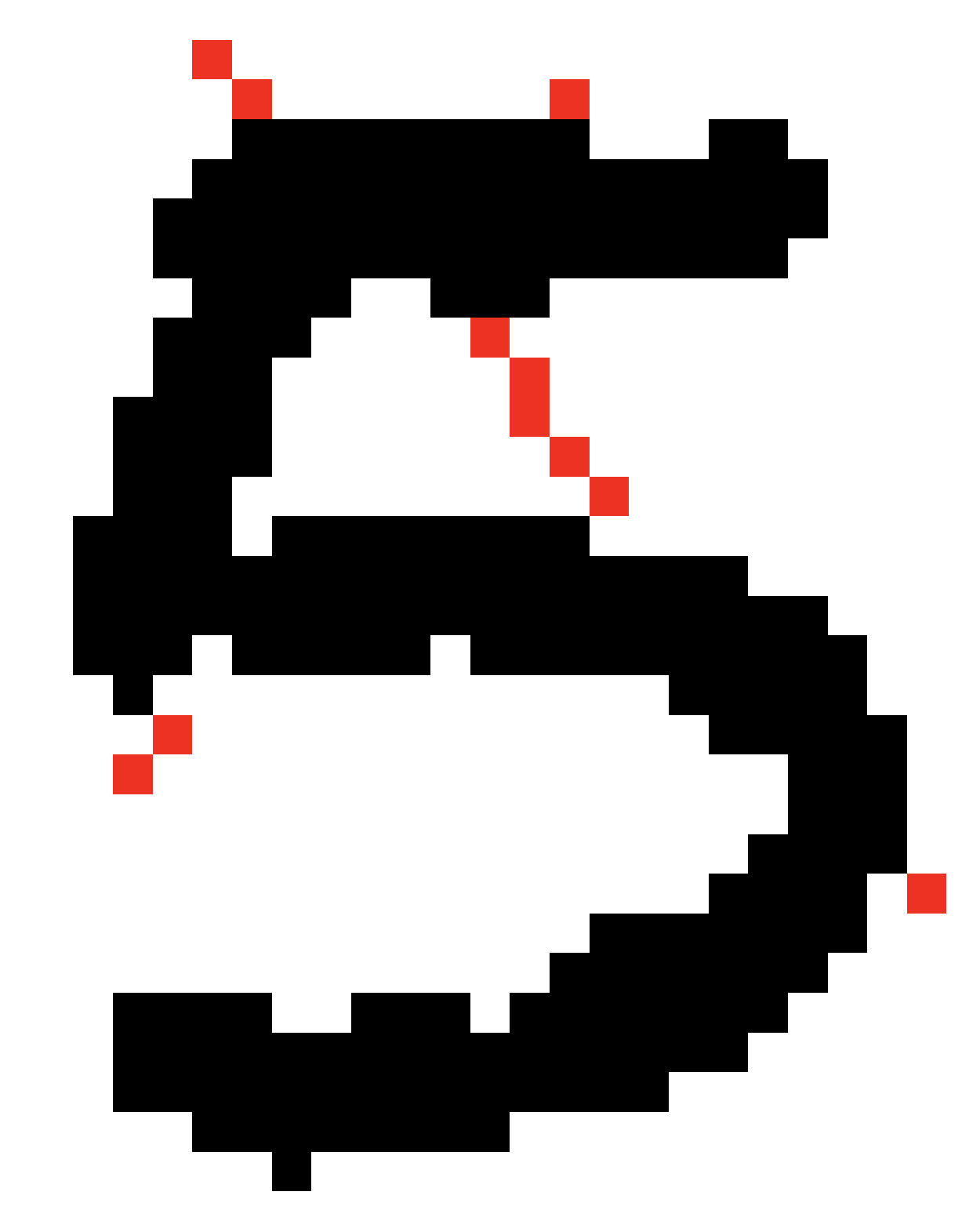}
	\hspace{3cm}
	\includegraphics[width=.09\linewidth]{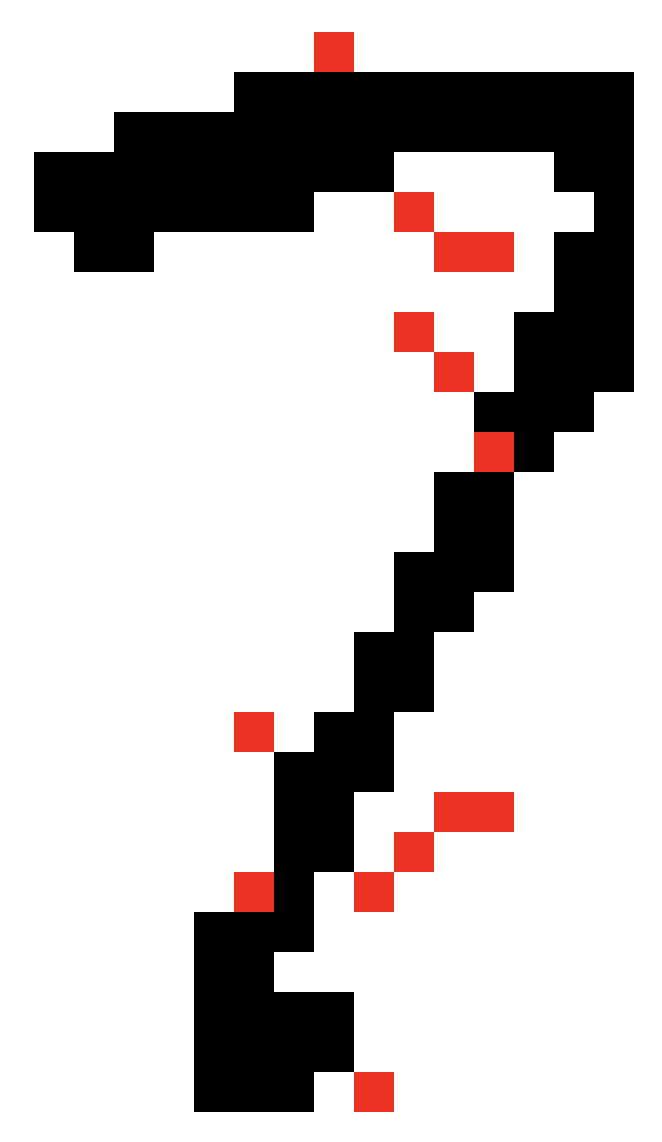}
	\hspace{3.5cm}
	\includegraphics[width=.13\linewidth]{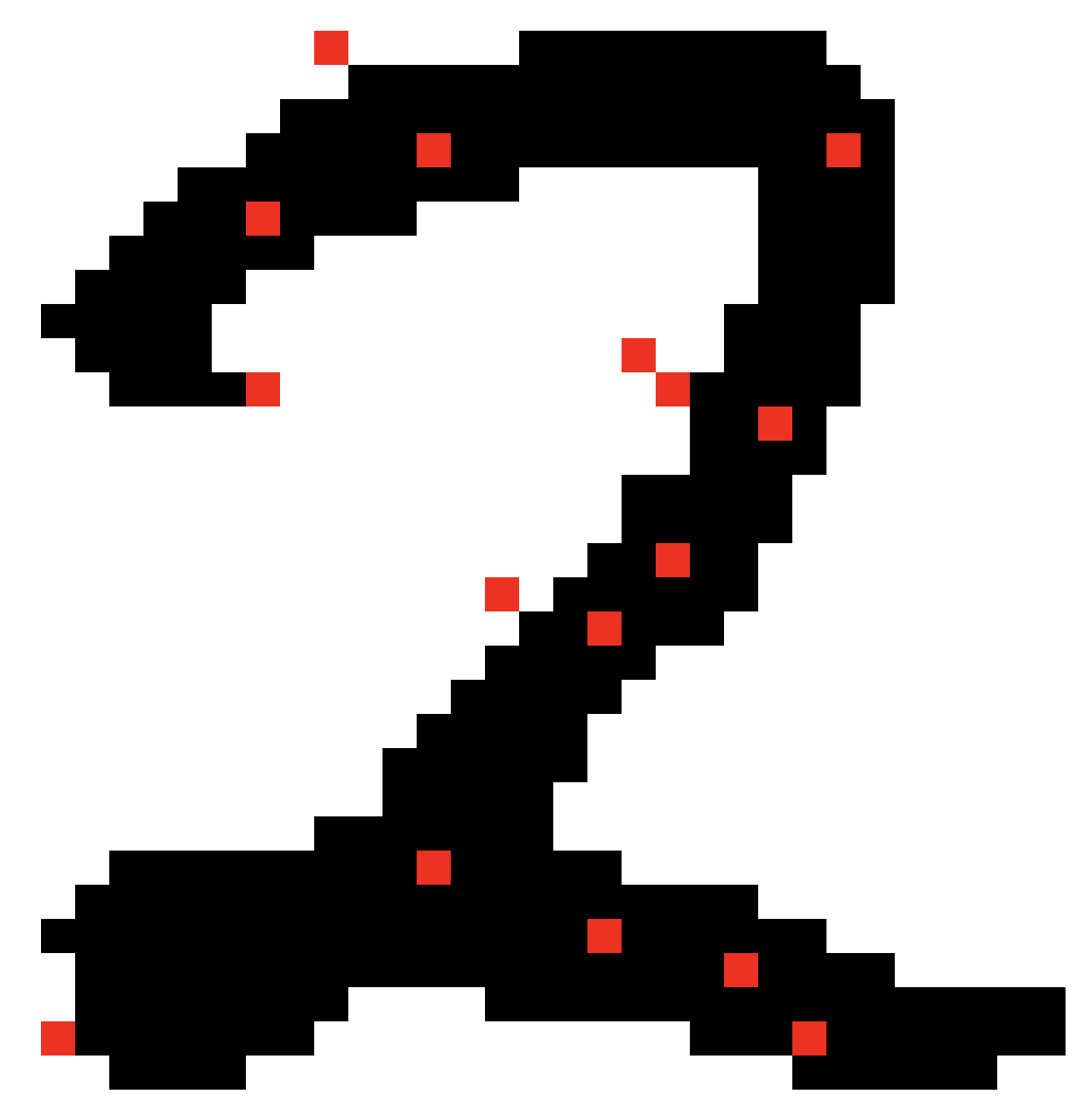}
	\includegraphics[width=.8\linewidth]{blank.png}
	\includegraphics[width=.2\linewidth]{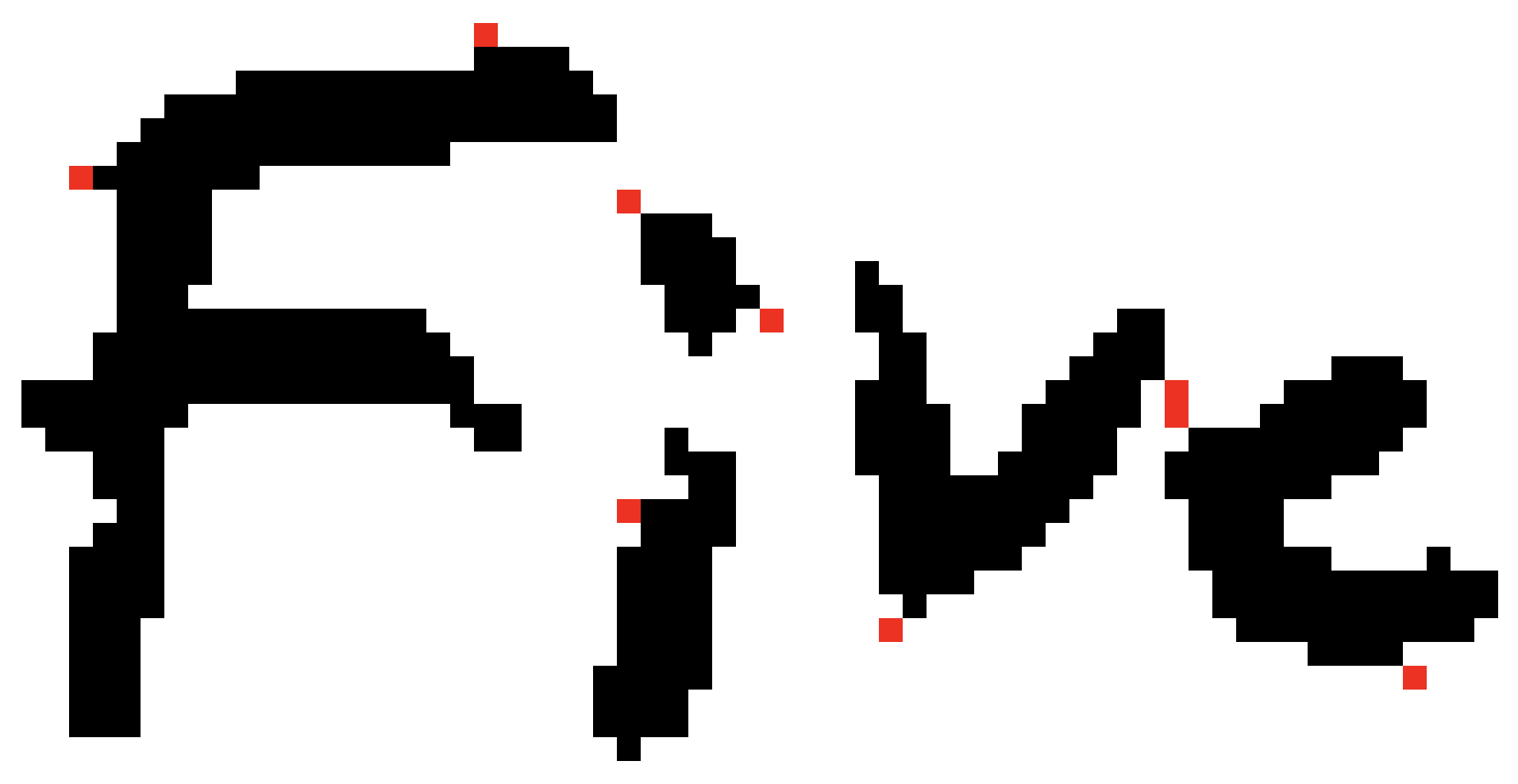}
	\hspace{1cm}
	\includegraphics[width=.3\linewidth]{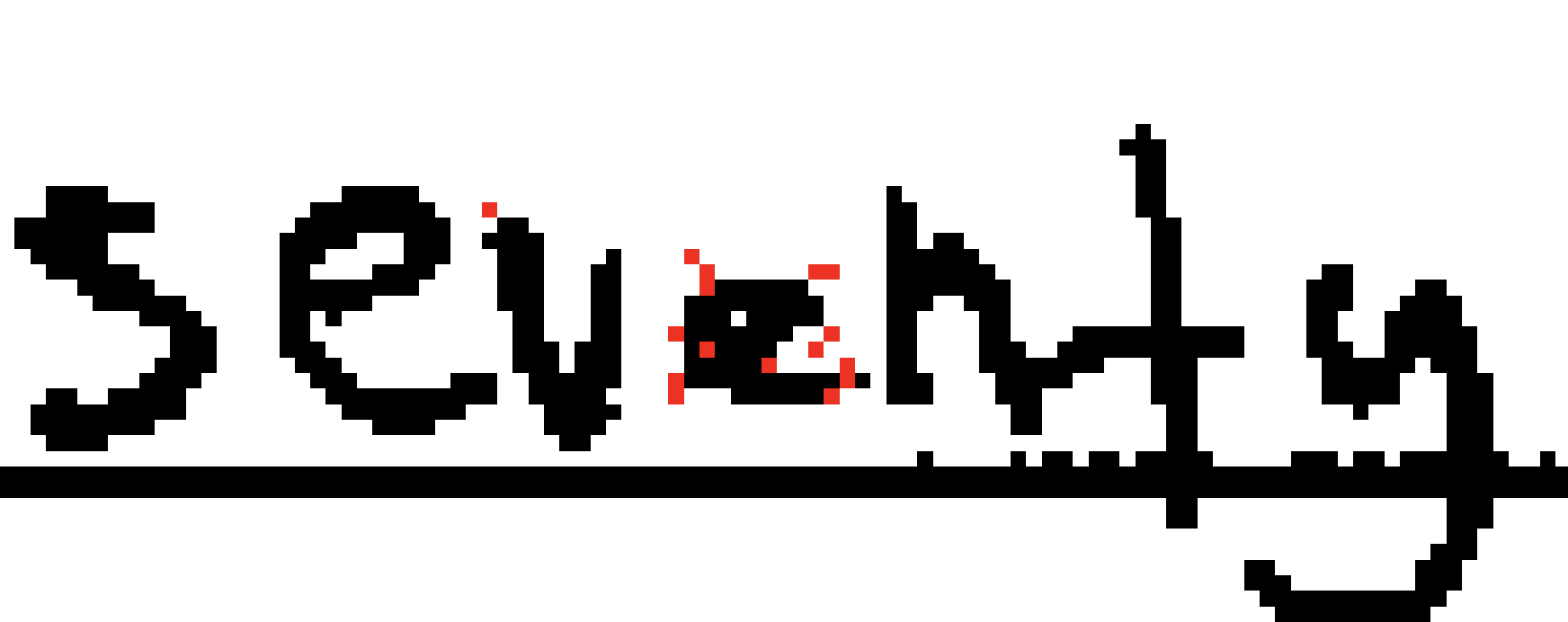}
	\hspace{1cm}
	\includegraphics[width=.25\linewidth]{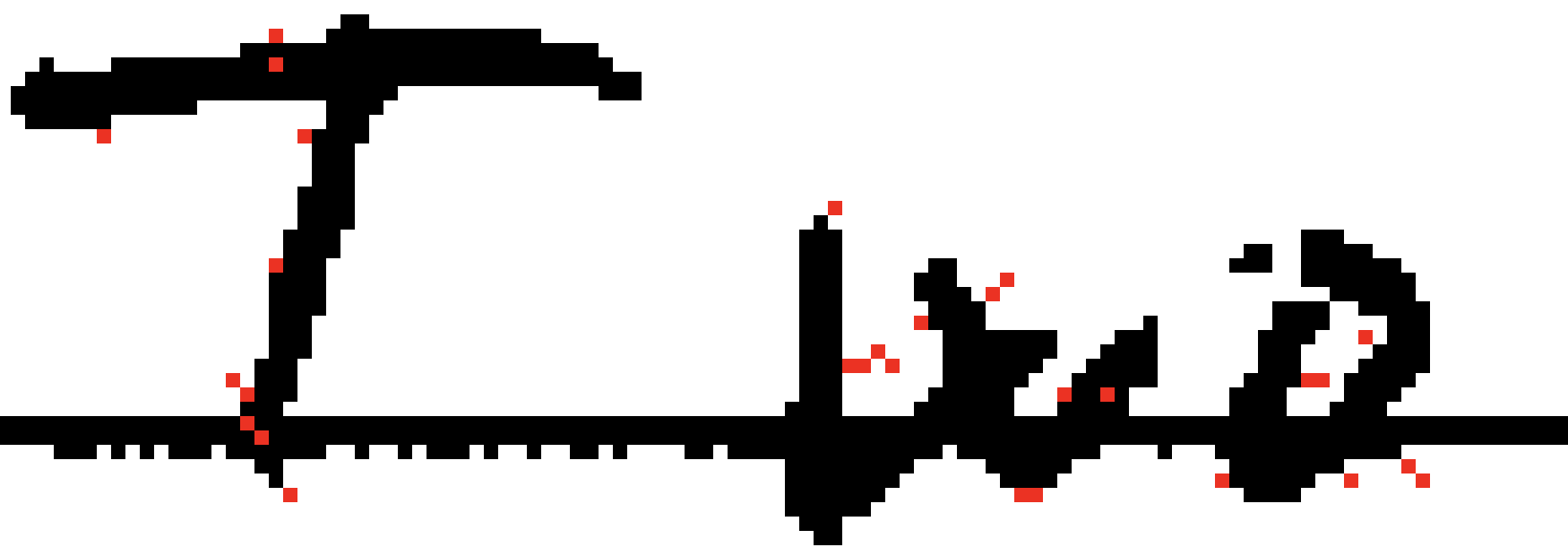}
	\label{fig:checks}
	\caption{First digit and word of the CAR and LAR amount of checks for $\$562$, $\$72$, and $\$2$ misclassified as $\$862$, $\$92$, and $\$3$ by a check processing system. The pixels in red correspond to pixels whose colors differ between the original and attacked image.}
\end{figure}

\end{document}